\documentclass[lettersize,journal]{IEEEtran}

\usepackage{amsmath,amsfonts,amssymb,amsthm}
\usepackage{array}
\usepackage{graphicx}
\usepackage{textcomp}
\usepackage{stfloats}
\usepackage{url}
\usepackage{cite}
\usepackage{verbatim}
\usepackage{color}
\usepackage{adjustbox}
\usepackage{multirow}
\usepackage{balance}
\usepackage{latexsym}
\usepackage[T1]{fontenc}
\usepackage[utf8]{inputenc}

\def\BibTeX{{\rm B\kern-.05em{\sc i\kern-.025em b}\kern-.08em
    T\kern-.1667em\lower.7ex\hbox{E}\kern-.125emX}}

\usepackage{caption}
\usepackage{subcaption}
\usepackage{soul}

\newtheorem{theorem}{Theorem}
\newtheorem{proposition}{Proposition}
\newtheorem{definition}{Definition}

\DeclareMathAlphabet\mathbfcal{OMS}{cmsy}{b}{n}

 \usepackage{algorithm,algpseudocode}
\algnewcommand\algorithmicinput{\textbf{Input:}}
\algnewcommand\algorithmicoutput{\textbf{Output:}}
\algnewcommand\Input{\item[\algorithmicinput]}%
\algnewcommand\Output{\item[\algorithmicoutput]}%
\algrenewcommand\algorithmicrequire{\textbf{Input:}}
\algrenewcommand\algorithmicensure{\textbf{Output:}}

\title{Subspace Clustering of Subspaces: Unifying Canonical Correlation Analysis and Subspace Clustering}

\author{Paris~A.~Karakasis,~\IEEEmembership{Graduate Student Member,~IEEE} and Nicholas~D.~Sidiropoulos,~\IEEEmembership{Fellow,~IEEE}%
\thanks{Paris~A.~Karakasis and Nicholas~D.~Sidiropoulos are with the Department of Electrical and Computer Engineering, University of Virginia, Charlottesville, VA 22904 USA (e-mail: $\left\{\text{karakasis,~nikos}\right\}$@virginia.edu).}
}

\begin{document}
\setlength{\belowdisplayskip}{2pt}
\setlength{\belowdisplayshortskip}{2pt}
\setlength{\abovedisplayskip}{2pt}
\setlength{\abovedisplayshortskip}{2pt}
\setlength{\textfloatsep}{5pt}
\maketitle
\begin{abstract}
We introduce a novel framework for clustering a collection of tall matrices based on their column spaces, a problem we term Subspace Clustering of Subspaces (SCoS). Unlike traditional subspace clustering methods that assume vectorized data, our formulation directly models each data sample as a matrix and clusters them according to their underlying subspaces. We establish conceptual links to Subspace Clustering and Generalized Canonical Correlation Analysis (GCCA), and clarify key differences that arise in this more general setting. Our approach is based on a Block Term Decomposition (BTD) of a third-order tensor constructed from the input matrices, enabling joint estimation of cluster memberships and partially shared subspaces. We provide the first identifiability results for this formulation and propose scalable optimization algorithms tailored to large datasets. Experiments on 
real-world hyperspectral imaging datasets demonstrate that our method achieves superior clustering accuracy and robustness, especially under high noise and interference, compared to existing subspace clustering techniques. These results highlight the potential of the proposed framework in challenging high-dimensional applications where structure exists beyond individual data vectors.
\end{abstract}

\noindent {\bf Keywords: Subspace Clustering, Partially Common / Shared Subspaces, Generalized Canonical Correlation Analysis, Block Term Tensor Decomposition}

\section{Introduction}

Discovering latent structures and patterns in real life datasets is critical for obtaining a better understanding of the principles that govern their generation processes, but also for designing successful solutions to modern and challenging engineering problems. When a dataset is high dimensional or consists of millions of samples, forming interpretations and revealing underlying relations between different features and different samples becomes even more challenging. Machine learning tools can be employed to facilitate the tasks at hand by finding low-dimensional representations of the high dimensional samples or by finding a small number of informative samples using clustering techniques.

Subspace Clustering (SC) is a well-studied problem where both dimensionality reduction and clustering are considered jointly. 
Traditional clustering techniques often struggle to identify meaningful clusters in high-dimensional spaces, as distance measures tend to become increasingly ineffective as the dimensionality of the considered sample space increases \cite{parsons2004subspace}. Such ineffectiveness becomes even more pronounced when irrelevant or unstructured features are present in a dataset. On the positive side, in many applications in image processing and computer vision, high-dimensional datasets can often be effectively embedded, as well as efficiently clustered, within diverse combinations of low-dimensional subspaces \cite{you2016oracle, you2016scalable}. Therefore, it is worthwhile to explore extensions of traditional clustering methods to identify clusters within different subspaces of the ambient space of a dataset. The generative model that aligns with the above concept is known as the Union of Subspaces (UoS) model \cite{vidal2011subspace}.

More specifically, the UoS model specifies a generative process in which vector-valued samples are drawn from an unknown union of finite subspaces. However, in several applications, we encounter datasets where each sample consists of a collection of relatively few high-dimensional vectors drawn from the same low dimensional subspace, while different samples correspond to potentially different subspaces. Such cases are common in cellular systems, where signals of different cell-edge users are received by surrounding multi-antenna base stations \cite{ibrahim2019cell}. They also appear in Wireless Sensor Networks (WSN), where different nodes record multichannel signals that may share some components \cite{hovine2022maxvar}. In these cases, effective clustering of the emerging matrix-valued samples would allow us to successfully detect and decode cell-edge users, whose received signals are typically weak \cite{ibrahim2019cell}. Additionally, it could enable the design of more efficient protocols for WSNs that exploit the presence of such clusters of nodes to further reduce the bandwidth usage and computational complexity at each node \cite{vlajic2006wireless, xia2007near, hovine2022maxvar}.

Popular subspace clustering techniques could be still applied in these cases to determine the sample memberships to different clusters by vectorizing each sample and processing the resulting sample vectors. However, this approach is not always applicable -- e.g., when the matrix-valued samples do not have the same number of columns. Furthermore, estimating the bases of the ambient subspaces cannot be directly achieved with the aforementioned strategy, as we explain in the sequel. Indeed, additional and costly post-processing steps would be required to obtain final estimates of the ambient subspaces. Moreover, the performance of these estimators in terms of clustering accuracy significantly degrades in the presence of noise, as will be confirmed by our experiments. 

Given these limitations and the challenges associated with applying traditional SC methods to matrix-valued samples, it is essential to consider special cases where the problem reduces to something that can be more  readily solved. When the number of clusters is equal to one, SC reduces to Principal Component Analysis (PCA). Interestingly, in the context of SC of matrix-valued samples and the extreme case where the number of clusters is also one, the problem of interest simplifies to estimating the common subspace spanned by the columns of most, if not all, samples. It was recently shown in \cite{ibrahim2019cell} that, in the case of two samples, this problem can be equivalently posed as another well-known problem: Canonical Correlation Analysis (CCA)\cite{hotelling1992relations}. The extension to more than two samples (referred to as views in the CCA literature) is known as Generalized Canonical Correlation Analysis (GCCA) \cite{horst1961generalized, kettenring1971canonical}, while conditions for recovering the underlying common subspace using GCCA can be found in \cite{sorensen2021generalized}.
Therefore, it seems natural to frame the problem of SC of matrix-value samples in a manner that organically extends GCCA.

In our prior work, published as a conference paper in  \cite{karakasis2023clustering}, we showed that such an extension is possible. Specifically, we proposed a problem formulation that builds upon a well-known tensor factorization model, the Block Term Decomposition (BTD) \cite{de2008decompositions} and an iterative optimization algorithm for learning the cluster memberships and a collection of bases of the partially ``common" subspaces in the ambient space. We compared the proposed approach to several state of the art subspace clustering methods using synthetic data and we demonstrated that the proposed method achieves high performance, even in the presence of strong noise where it significantly outperforms the prior art. 

In this journal manuscript, we extend our earlier work in \cite{karakasis2023clustering} in several important directions. First, we provide the first identifiability result for the proposed model, offering theoretical insights into the conditions under which recovery is possible. Second, we refine the optimization algorithm to enhance convergence properties and improve scalability to larger datasets. Third, we propose an alternative problem formulation and algorithmic variant designed to handle cases where subspaces exhibit significant overlap, addressing a common challenge in real-world applications. Furthermore, we introduce a strategy for determining the number of clusters, and we also propose  methods to estimate the dimensions of the partially common subspaces. Finally, we demonstrate the practical relevance of our approach through a new and noteworthy application to hyperspectral image analysis, focusing on pixel-level clustering. We motivate this setting through a generative model and evaluate our method on real hyperspectral datasets. Across three real-life benchmark datasets, our method achieves performance comparable to state-of-the-art techniques on two, while significantly outperforming them on the third.

\textbf{Notation:} Scalars are denoted by lowercase letters (e.g., $x$), vectors by bold lowercase letters (e.g., $\mathbf{x}$), matrices by bold uppercase letters (e.g., $\mathbf{X}$), and tensors by bold calligraphic letters (e.g., $\mathbfcal{X}$). The $i$-th entry of a vector $\mathbf{x}$ is denoted by $x_i$, and the $(i,j)$-th entry of a matrix $\mathbf{X}$ by $\mathbf{X}\left(i,j\right)$. The transpose of a matrix $\mathbf{X}$ is denoted by $\mathbf{X}^\top$. The Frobenius norm is denoted by $\|\cdot\|_F$, and the Euclidean (or $\ell_2$) norm by $\|\cdot\|_2$. The identity matrix of size $N$ is denoted by $\mathbf{I}_N$. The Kronecker product is denoted by $\otimes$ and the outer product by $\circ$. The set of real $M \times N$ matrices is denoted by $\mathbb{R}^{M \times N}$. The Hadamard (elementwise) product of two matrices is denoted by $\ast$. The span (column space) of a matrix $\mathbf{X}$ is written as $\mathrm{col}(\mathbf{X})$, and the projector matrix of a subspace $\mathbb{S}$ as $\mathbf{P}_{\mathbb{S}}$.

\section{Preliminaries}

In this section, we review some basic concepts and results that will serve as the foundation for presenting the proposed framework in the subsequent sections. In the first subsection, we introduce one of the most popular problem formulations of GCCA, known as the Maximum Variance (MAXVAR) formulation. The second subsection covers the concept of the distance between two subspaces and specifically what is known as chordal distance. Finally, we introduce the concept of a tensor and discuss the Block Term Decomposition (BTD) model \cite{de2008decompositions}, which will serve as a basis for our subsequent analysis.

\subsection{The MAXVAR formulation of Generalized CCA}

Let $\left\{\mathbf{X}_k\right\}_{k=1}^K$ be a collection of tall matrices, where $\mathbf{X}_k\in\mathbb{R}^{N\times M_k}$, for $k \in \left[K\right]:=\{1,\dots,K\}$, with $N>\max_k M_k$. The MAXVAR formulation of GCCA seeks matrices $\mathbf{Q}_k\in\mathbb{R}^{M_k\times L}$ and $\mathbf{G}\in\mathbb{R}^{N\times L}$ that solve the problem
\begin{equation}
\begin{split}
    \min_{\left\{\mathbf{Q}_k\right\}_{k=1}^K, \mathbf{G}} &~\sum_{k=1}^K\left\|\mathbf{X}_k\mathbf{Q}_k-\mathbf{G}\right\|_F^2\\
    \text{s.t.} &~~\mathbf{G}^T\mathbf{G}=\mathbf{I}_L,
    \label{MAXVAR}
\end{split}
\end{equation}
and it can be interpreted as finding an $L$-dimensional ``common'' subspace in $\mathbb{R}^N$ that lies in the intersection of (or is as close as possible to) all column spaces $\text{col}\left(\mathbf{X}_k\right)$. Conditions and guarantees under which solving the GCCA problem identifies the common underlying subspace across all views in the noiseless case can be found in \cite{ibrahim2019cell} and \cite{sorensen2021generalized} for the two-view and multiview cases, respectively. 


\subsection{Squared Chordal Distance of a Pair of Subspaces}

Several distances between pairs of subspaces have been proposed over the years. For an extensive coverage of the topic, we refer the interested reader to \cite{ye2016schubert, stewart1990matrix}. In this work, we focus on the chordal distance of two subspaces. Given two $p$-dimensional subspaces of $\mathbb{R}^N$, $\mathbb{S}_1^p$ and $\mathbb{S}_2^p$, the chordal distance between them is defined as \cite{ye2016schubert}
\begin{equation*}
    d_{Gr(p, n)} := \left(\sum_{i=1}^p\sin^2\left(\theta_i\right)\right)^{1/2}=\frac{1}{\sqrt{2}}\left\|\mathbf{P}_{\mathbb{S}_1^p} - \mathbf{P}_{\mathbb{S}_2^p}\right\|_F,
\end{equation*}
where $\theta_i$ denotes the $i$-th principal angle between $\mathbb{S}_1^p$ and $\mathbb{S}_2^p$, while $\mathbf{P}_{\mathbb{S}_1^p}$ and $\mathbf{P}_{\mathbb{S}_2^p}$ denote the projector matrices of $\mathbb{S}_1^p$ and $\mathbb{S}_2^p$, respectively.

Considering the squared chordal distance between two subspaces allows us to define a differentiable metric for measuring the difference between them. Although this metric does not qualify as a proper distance function - since it does not satisfy the triangle inequality - it remains highly useful in practical applications, as we will demonstrate in the subsequent sections. Furthermore, this metric is well-defined even when dealing with subspaces of potentially different dimensions. Without loss of generality, let $\mathbb{S}_1^{p_1}$ and $\mathbb{S}_2^{p_2}$ be two subspaces of $\mathbb{R}^N$, where $p_1\leq p_2$ denote their dimensions, respectively. Then, their squared chordal distance is given by
\begin{equation*}
    d_{Gr(p_1, p_2, n)}^2 := \frac{p_2-p_1}{2}+\sum_{i=1}^{p_1}\sin^2\left(\theta_i\right)=\frac{1}{2}\left\|\mathbf{P}_{\mathbb{S}_1^{p_1}} - \mathbf{P}_{\mathbb{S}_2^{p_2}}\right\|_F^2
\end{equation*}
and vanishes only when $p_1 = p_2$ and $\mathbb{S}_1^{p_1}=\mathbb{S}_2^{p_2}$.

\subsection{The Block Term Decomposition}

So far, we have focused on vectors and matrices. A tensor is a mathematical concept that generalizes scalars, vectors, and matrices to higher dimensions. Specifically, scalars are categorized as $0$-order tensors, vectors as $1$-order tensors, and matrices as $2$-order tensors. In general, a tensor can be seen as a multidimensional array, while the order of a tensor is determined by the number of indices required to refer to its elements. For example, three indices are required to access the elements of a tensor $\mathbfcal{X}\in \mathbb{R}^{I\times J\times K}$, $\mathbfcal{X}\left(i,j,k\right)$, where $i\in\left[I\right]$, $j\in\left[J\right]$, and $k\in\left[K\right]$. Therefore, $\mathbfcal{X}$ is a third order tensor, which is the highest order of tensors that we will consider in this work. For a more detailed coverage of tensors, we refer the interested reader to \cite{sidiropoulos2017tensor}.

Similarly to matrices, higher order tensors can be decomposed into simpler tensors.
For instance, the concept of a rank-one matrix extends naturally to that of a rank-one tensor. Moreover, any tensor can be expressed as a finite sum of rank-one tensors, which leads to what is known as the Canonical Polyadic Decomposition (CPD) of a tensor. In this work, however, we focus on a different tensor decomposition model known as Block Term Decomposition. 

We say that a tensor $\mathbfcal{X}\in \mathbb{R}^{I\times J\times K}$ admits a Block Term Decomposition (BTD) in rank-$(L_r,L_r,1)$ terms \cite{de2008decompositions}, if there exists a collection of matrices $\mathbf{A}_r\in\mathbb{R}^{I\times L_r}$, $\mathbf{B} _{r}\in\mathbb{R}^{J\times L_r}$, for each $r\in\left[R\right]$, and a matrix $\mathbf{C}\in\mathbb{R}^{K\times R}$, such that
\begin{equation}
\boldsymbol{\mathcal{X}}\left(i,j,k\right)=\sum_{r=1}^R\left(\sum_{l_r=1}^{L_r}\mathbf{A}_r\left(i,l_r\right)\mathbf{B}_r\left(j,l_r\right)\right)\mathbf{C}\left(k, r\right).
\end{equation}
Using $\circ$ to denote the outer product between a matrix and a vector, the above relation can be compactly written as   
\begin{equation}
\boldsymbol{\mathcal{X}}=\sum_{r=1}^R\left(\mathbf{A}_r\mathbf{B}_r^T\right)\circ\mathbf{c}_{r}=:\left(\mathbf{A},\mathbf{B},\mathbf{C}\right),
\end{equation}
where $\mathbf{c}_{r}$ denotes the $r$-th column of matrix $\mathbf{C}$, while matrices $\mathbf{A}$ and $\mathbf{B}$ are defined as $\mathbf{A}:=\left[\mathbf{A}_1,\ldots,\mathbf{A}_R\right]\in\mathbb{R}^{I\times\sum_{r=1}^RL_r}$ and $\mathbf{B}:=\left[\mathbf{B}_1,\ldots,\mathbf{B}_R\right]\in\mathbb{R}^{J\times\sum_{r=1}^RL_r}$. The notation presented in the last equality, $\left(\mathbf{A},\mathbf{B},\mathbf{C}\right)$, will be used from now on to denote a BTD of a given tensor. In general, we are interested in decompositions that are valid for the smallest possible values of $R$ and $L_r$, for all $r$. Several sufficient conditions that an obtained decomposition has to satisfy in order to guarantee \emph{essential uniqueness}—that is, uniqueness up to block-wise invertible transformations and permutations of the rank-$(L_r, L_r, 1)$ terms—have been proposed in \cite{de2008decompositions}.

For tensors with symmetric frontal slabs, i.e. $\mathbfcal{X}(i, j, k)=\mathbfcal{X}(j, i, k)$, for all $(i, j, k)\in\left[I\right]\times\left[I\right]\times \left[K\right]$, it makes sense to consider a BTD model where $\mathbf{A}_r=\mathbf{B}_{r}$, for all $r$. The above restriction yields the symmetric BTD model
\begin{equation}
\boldsymbol{\mathcal{X}}=\sum_{r=1}^R\left(\mathbf{A}_r\mathbf{A}_r^T\right)\circ\mathbf{c}_{r}=\left(\mathbf{A},\mathbf{A},\mathbf{C}\right).
\end{equation}
This model will form the basis of the proposed approach presented in this paper. 
In the case where the factors $\mathbf{A}_r$ are constrained to be columnwise orthonormal, i.e., $\mathbf{A}_r^T \mathbf{A}_r = \mathbf{I}_{L_r}$, and the elements of $\mathbf{C}$ are nonnegative, essential uniqueness implies that any two decompositions of the same tensor, $(\mathbf{A}, \mathbf{A}, \mathbf{C})$ and $(\tilde{\mathbf{A}}, \tilde{\mathbf{A}}, \tilde{\mathbf{C}})$, must coincide up to permutation and unitary transformations. That is, after aligning the rank-one components via an appropriate permutation, the corresponding factors satisfy $\tilde{\mathbf{A}}_r = \mathbf{A}_r \mathbf{E}_r$, where $\mathbf{E}_r \in \mathbb{R}^{L_r \times L_r}$ is unitary, and $\tilde{\mathbf{C}} = \mathbf{C}$. Hence, in terms of factors $\mathbf{A}_r$, essential uniqueness translates to decompositions where the column spaces of corresponding factors $\mathbf{A}_r$, $\text{col}\left(\mathbf{A}_r\right)$ are the same.

\section{Theory}

Let $\left\{\mathbf{X}_k\right\}_{k=1}^K$ be a collection of full column-rank matrices, where $\mathbf{X}_k\in\mathbb{R}^{N\times M_k}$ for $k\in\left[K\right]$ and $N>\max_k M_k$. As discussed earlier, solving the GCCA problem can be used for estimating a common subspace that is spanned by the columns of each matrix $\mathbf{X}_k$. In this section, we consider the case where non-overlapping subsets of matrices $\mathbf{X}_k$ span partially common subspaces with their columns. Based on this assumption, we propose a generative model for the matrices $\left\{\mathbf{X}_k\right\}_{k=1}^K$ and formulate a problem for identifying both the assignment of these matrices to the different partially common subspaces and the partially common subspaces themselves.

\subsection{Proposed Generative Model}
Consider a partition of the set $\left[K\right]$, $\left\{\mathbb{I}_r\right\}_{r=1}^R$, such that
$\cup_{r=1}^R\mathbb{I}_r=\left[K\right]~\text{ and }~\mathbb{I}_r\cap\mathbb{I}_s=\emptyset$ if $r\neq s$. We assume that for a fixed $r\in\left[R\right]$ and all $k\in \mathbb{I}_r$, matrices $\mathbf{X}_k$ are generated according to
\begin{equation}
    \mathbf{X}_k = \begin{bmatrix} \mathbf{G}_r & \mathbf{H}_k \end{bmatrix}\mathbf{Q}_k,
    \label{GenModel}
\end{equation}
where $\mathbf{G}_r\in\mathbb{R}^{N\times L_r}$ denotes an orthonormal basis of a shared $L_r$-dimensional subspace, with $L_r\leq \min_{k\in\mathbb{I}_r}M_k$, $\mathbf{H}_k\in\mathbb{R}^{N\times M_k-L_r}$ denotes an orthonormal basis of the relative complement of $\text{col}\left(\mathbf{G}_r\right)$ with respect to $\text{col}\left(\mathbf{X}_k\right)$, and $\mathbf{Q}_k\in\mathbb{R}^{M_k\times M_k}$ is full-rank. A direct consequence of our assumptions is that 
$\cap_{k\in\mathbb{I}_r}\text{col}\left(\mathbf{H}_k\right)=\emptyset$, and as a result $\cap_{k\in\mathbb{I}_r}\text{col}\left(\mathbf{X}_k\right)=\text{col}\left(\mathbf{G}_r\right)$.
As can be seen, the proposed generative model generalizes the generative model proposed in \cite{sorensen2021generalized}, under which all matrices $\mathbf{X}_k$ share a common subspace and hence subspaces $\text{col}\left(\mathbf{G}_r\right)$ are equal, for all $r\in\left[R\right]$.

The proposed generative model is closely related to the UoS model of vector-valued samples. This can be verified after observing that when $M_k=M$ for all $k\in\left[K\right]$, we get
\begin{equation}
\begin{split}
    \mathbf{x}_k&:=\text{vec}\left(\mathbf{X}_k\right)\\
&=\text{vec}\left(\mathbf{G}_r\mathbf{Q}_k^{(1)}\right)+ \text{vec}\left(\mathbf{H}_k\mathbf{Q}_k^{(2)}\right)\\
    &=\left(\mathbf{I}_M\otimes \mathbf{G}_r\right)\text{vec}\left(\mathbf{Q}_k^{(1)}\right)+ \left(\mathbf{I}_M\otimes \mathbf{H}_k\right)\text{vec}\left(\mathbf{Q}_k^{(2)}\right)\\ &=\left(\mathbf{I}_M\otimes \mathbf{G}_r\right)\mathbf{q}_k^{(1)}+ \left(\mathbf{I}_M\otimes \mathbf{H}_k\right)\mathbf{q}_k^{(2)},
    \raisetag{20mm}
    \label{UoS}
\end{split}
\end{equation}
where $\mathbf{Q}_k^{(1)}\in\mathbb{R}^{L_r\times M}$ and $\mathbf{Q}_k^{(2)}\in\mathbb{R}^{M-L_r\times M}$, such that $\mathbf{Q}_k =\left[\mathbf{Q}_k^{(1)^T},~\mathbf{Q}_k^{(2)^T}\right]^T$, $\mathbf{q}_k^{(s)}=\text{vec}\left(\mathbf{Q}_k^{(s)}\right)$ for $s=1,2$, and $\otimes$ denotes the Kronecker product between two matrices.

When $\mathbf{Q}_k^{(2)}=\mathbf{0}_{M-L_r\times M}$ for all $k$, the model of (\ref{UoS}) matches the Union of Linear Subspaces\footnote{Here we specify that we have a Union of Linear Subspaces as there exists also the variation of the Union of Affine Subspaces (see for example in \cite{vidal2011subspace}).} generative model that is usually met in the SC literature \cite{vidal2011subspace}, but with structured subspaces due to the appearance of the Kronecker product. Recovering a basis of the subspace $\text{col}\left(\mathbf{G}_r\right)$ from a basis of the subspace $\text{col}\left(\mathbf{I}_M\otimes \mathbf{G}_r\right)$ is not trivial, as it is equivalent to solving a combination of the constrained Nearest Kronecker Product problem \cite{golub2013matrix} and the Orthogonal Procrustes Problem \cite{schonemann1966generalized}. Again, this two-step approach, for obtaining the cluster assignments and estimates of $\text{col}\left(\mathbf{G}_r\right)$, is possible only in the case where $M_k=M$ for all $k\in \left[K\right]$. In the following subsection, we propose a direct approach for obtaining the cluster memberships and the estimates of $\text{col}\left(\mathbf{G}_r\right)$, for the general case where $M_k$ are not necessarily equal.

\subsection{Proposed Problem Formulations}

Since we are only interested in the column spaces of each matrix $\mathbf{X}_k$, we may work with any orthonormal basis that spans the same subspace. In the remainder of this paper, we denote such orthonormal bases by $\mathbf{U}_k\in\mathbb{R}^{N\times M_k}$, i.e. $\text{col}\left(\mathbf{X}_k\right)=\text{col}\left(\mathbf{U}_k\right)$, for all $k$. Then, we may construct for each view the orthogonal projection matrices $\mathbf{P}_{\text{col}\left(\mathbf{X}_k\right)}=\mathbf{U}_k\mathbf{U}_k^T$ and consider the frontal symmetric tensor $\mathbfcal{P}\in\mathbb{R}^{N \times N\times K}$ given by
\begin{equation}
\mathbfcal{P}(:,:, k):=\mathbf{P}_{\text{col}\left(\mathbf{X}_k\right)}=\mathbf{U}_k\mathbf{U}_k^T, 
\label{Ptensor}
\end{equation}
for $k=1,\ldots, K$.

After assuming the generative model in (\ref{GenModel}), one may observe that 
\begin{equation}
\mathbfcal{P}(:,:, k) = \underbrace{\sum_{r=1}^R \mathbf{c}_r(k)\mathbf{G}_r\mathbf{G}_r^T}_{\mathbfcal{Z}(:,:, k)} + \underbrace{\mathbf{H}_k\mathbf{H}_k^T}_{\mathbfcal{V}(:,:, k)},~~~\forall k\in\left[K\right],
\label{TensorModel}
\end{equation}
where $\mathbf{c}_r$ denotes the $r$-th column of the assignment matrix $\mathbf{C}\in\left\{0,1\right\}^{K\times R}$ encoding which of the $R$ clusters does matrix $\mathbf{X}_k$ belong to. Moreover, one may notice that the symmetric BTD of tensor $\mathbfcal{Z}$, $\left(\mathbf{G},\mathbf{G}, \mathbf{C}\right)$, appears in (\ref{TensorModel}). According to the proposed generative model, each matrix $\mathbf{X}_k$ belongs to one and only one cluster. As a result, one may easily verify that the off-diagonal elements of the product $\mathbf{C}^T\mathbf{C}$ will be zero, while its $r$-th diagonal element will be equal to the number of matrices $\mathbf{X}_k$ that belong to the $r$-the cluster.

Next, we consider the problem of recovering a symmetric BTD of tensor $\mathbfcal{Z}$ from tensor $\mathbfcal{P}$. To obtain such an approximation, we consider the following optimization problem
\begin{equation}
\begin{split}
\min_{\widehat{\mathbf{G}}, ~\widehat{\mathbf{C}}} &~~ f_{\mathbfcal{P}}\left(\widehat{\mathbf{G}}, ~\widehat{\mathbf{C}}\right)\hspace{1mm}:=~\frac{1}{2}\left\|\mathbfcal{P}-\sum_{r=1}^R\left(\hat{\mathbf{G}}_r\hat{\mathbf{G}}_r^T\right)\circ\hat{\mathbf{c}}_r\right\|_F^2 \\
\text{s.t. }&~~~~~~~ \hat{\mathbf{G}}_r^T\hat{\mathbf{G}}_r=\mathbf{I}_{L_r}, ~~\forall r\in\left[R\right]\\
&~~~~~~~\hat{\mathbf{C}}\in \left\{0,1\right\}^{K\times R},\\
&~~~~~~~\hat{\mathbf{C}}(:,i)^T\hat{\mathbf{C}}(:,j)=0,~\forall~i\neq j.
\label{ProbBasic}
\end{split}
\end{equation}
Our choice of the approximation metric, apart from the fact that the Frobenious norm is a very popular metric in the tensor decomposition literature, can also be motivated by observing that the Frobenius norm can be written as
$$\sum_{k=1}^K~\frac{1}{2}\left\|\mathbfcal{P}(:,:,k)-\sum_{r=1}^R\hat{\mathbf{C}}(k,r)\hat{\mathbf{G}}_r\hat{\mathbf{G}}_r^T\right\|_F^2.$$
Because of our assumptions, for each $k$ only one element of the corresponding row of matrix $\hat{\mathbf{C}}$ will be nonzero, and more specifically it will be equal to one. As a result, the resulting cost function will be a sum of $K$ squared chordal distances between pairs of subspaces.

The combination of binary and column-wise orthogonality constraints, imposed on factor $\hat{\mathbf{C}}$, results in a cumbersome optimization problem. Relaxing the binary constraints to nonnegativity constraints could still result in a valid encoding of the assignment of each matrix $\mathbf{X}_k$ to one of the $R$ clusters, as long as the orthogonality constraints are satisfied. Moreover, the new combination of constraints can be handled in a very efficient and simple way by using the method proposed in \cite{li2014two}. This method relies on employing the penalty method and in our case translates into solving, for an increasing nonnegative sequence $\left\{\rho_t\right\}_{t=1}^{\infty}$, the subproblems
\begin{equation}
\begin{split}
\min_{\widehat{\mathbf{G}}, ~\widehat{\mathbf{C}}} &~~~ f_{\mathbfcal{P}}\left(\widehat{\mathbf{G}}, ~\widehat{\mathbf{C}}\right)  + \frac{\rho_t}{2}~\text{trace}\left(\mathbf{Q}\hat{\mathbf{C}}^T\hat{\mathbf{C}}\right),\\
\text{s.t. }&~~~~~\hat{\mathbf{C}}\geq 0,~ \hat{\mathbf{G}}_r^T\hat{\mathbf{G}}_r=\mathbf{I}_{L_r}, ~~\forall r\in\left[R\right],
\label{FitCri}
\raisetag{10.5mm}
\end{split}
\end{equation}
where $\mathbf{Q}=\mathbf{1}_R\mathbf{1}_R^T-\mathbf{I}_R$. In our prior work in \cite{karakasis2023clustering}, we propose solving each one of the subproblems in (\ref{FitCri}) sequentially and using the resulting solution as initialization for the next one, until the orthogonality constraints on factor $\hat{\mathbf{C}}$ are met. 

While the penalty method offers a straightforward approach to handling the orthogonality constraints on $\hat{\mathbf{C}}$, it often requires very large penalty parameters to enforce feasibility. This can lead to ill-conditioned optimization problems and slow convergence \cite{bertsekas2014constrained, nocedal1999numerical}. To address these drawbacks, in this work we propose the adoption of the augmented Lagrangian method \cite{hestenes1969multiplier, powell1969method}, which combines the benefits of penalty-based formulations with dual variable updates, thereby improving numerical stability and convergence behavior.
Specifically, after letting $\boldsymbol{\Lambda}_t\in\mathbb{R}^{R\times R}$ be the symmetric zero diagonal Lagrange multiplier matrix, we consider the augmented Lagrangian
\begin{equation}
\begin{split}
& f_{\mathbfcal{P}}\left(\widehat{\mathbf{G}}, ~\widehat{\mathbf{C}}\right) + \left\langle \boldsymbol{\Lambda}_t,\hspace{0.5mm}\widehat{\mathbf{C}}^T \widehat{\mathbf{C}} \right\rangle
+ \frac{\rho_t}{2} \left\| \left(\widehat{\mathbf{C}}^T \widehat{\mathbf{C}}\right)\ast \mathbf{Q} \right\|_F^2,
\raisetag{0mm}
\end{split}
\label{eq:aug_lagrangian}
\end{equation}
where $\left\langle\cdot,~\cdot\right\rangle$ denotes the inner product between two matrices and $\rho_t>0$ can be constant or a term of an increasing nonnegative sequence $\left\{\rho_t\right\}_{t=1}^{\infty}$, under nonnegativity constraints on $\widehat{\mathbf{C}}$ and columnwise orthogonality constraints on all matrices $\hat{\mathbf{G}}_r$.

In cases where the partially common subspaces exhibit significant overlap or the collected views are heavily contaminated with noise, algorithms designed for the problem formulations in (\ref{FitCri}) and (\ref{eq:aug_lagrangian}) may become ineffective, primarily due to the presence of the orthogonality constraints. On the one hand, interior-point methods that preserve orthogonality constraints can be highly sensitive and prone to getting stuck in poor local minima. On the other hand, deflation-based approaches that sequentially extract the different partially common subspace bases tend to perform poorly when the clusters lie in subspaces with significant overlap. Dropping the orthogonality constraints allows us to use efficient interior-point methods that update all bases concurrently, while not restricting them to the Stiefel manifold enables more flexible trajectories during optimization. 

The downside of this flexibility is that it introduces scaling ambiguity within each rank-one term of the tensor $\mathbfcal{Z}$. As a result, the regularizations imposed on $\mathbf{C}$ in (\ref{FitCri}) and (\ref{eq:aug_lagrangian}) become ineffective and require additional care. To address this issue, we propose applying the regularizations in (\ref{FitCri}) and (\ref{eq:aug_lagrangian}) to a column-wise Euclidean norm–normalized version of $\mathbf{C}$, denoted by $\mathbf{C}_{\mathrm{norm}}$, thereby leading to a scale-invariant regularization scheme. However, as the number of views, $K$, increases, the values of $\mathbf{C}_{\mathrm{norm}}$ tend to decrease, which makes penalizing the cross products of its columns more challenging to optimize. To that end, we define $\boldsymbol{\Psi}\left(\mathbf{C}\right):=\sqrt{\mathbf{C}_{norm}^T\mathbf{C}_{norm}+\varepsilon}$, where the square root is applied element-wise, and we propose the following variant of the augmented Langrangian in (\ref{eq:aug_lagrangian})
\begin{equation}
\begin{split}
& f_{\mathbfcal{P}}\left(\widehat{\mathbf{G}}, ~\widehat{\mathbf{C}}\right) + \left\langle \boldsymbol{\Lambda}_t,\hspace{0.5mm}\boldsymbol{\Psi}\left(\widehat{\mathbf{C}}\right) \right\rangle
+ \frac{\rho_t}{2} \left\|\boldsymbol{\Psi}\left(\widehat{\mathbf{C}}\right)\ast \mathbf{Q} \right\|_F^2,
\raisetag{2mm}
\end{split}
\label{eq:aug_lagrangian_II}
\end{equation}
under nonnegativity constraints on $\widehat{\mathbf{C}}$. The square root in $\Phi$ enlarges small values that may appear in the cross products of the columns of $\mathbf{C}_{\mathrm{norm}}$, thereby facilitating the penalization of assigning the same sample to more than one cluster. Considering the augmented Langrangian in (\ref{eq:aug_lagrangian_II}) enables tackling effectively problems of large numbers of views that are clustered in subspaces of significant overlap, as we show in our experiments.

Finally, note that after obtaining a solution to the problem in (\ref{eq:aug_lagrangian_II}), one can compute the projection of the positive semidefinite matrices $\widehat{\mathbf{G}}_r\widehat{\mathbf{G}}_r^T$ onto the set of scaled projection matrices in a principled manner, thereby obtaining an estimate of the partially common subspaces and automatically determining their dimensions.

\subsection{Identifiability}

In the previous subsections, we outline a roadmap that begins with the proposed generative model and leads to a problem formulation for recovering the cluster memberships and bases of the partially common subspaces. Solving the problems in (\ref{ProbBasic}), (\ref{FitCri}) and (\ref{eq:aug_lagrangian}), aim to recover a symmetric BTD of tensor $\mathbfcal{Z}$ from tensor $\mathbfcal{P}$. However, several questions arise here. 

As a first question, assume that we have access to tensor $\mathbfcal{Z}$ and one of its symmetric BTDs. Does this imply that we essentially have a valid encoding of the cluster memberships and bases of the partially common subspaces? The answer to this question relates to what is known as the essential uniqueness of a tensor decomposition. The second question pertains to the BTD of tensor $\mathbfcal{P}$: under which conditions is its symmetric BTD essentially unique? Next, we show under what conditions tensors $\mathbfcal{Z}$ and $\mathbfcal{P}$  admit unique tensor decompositions.

Before we investigate the two questions, we assume that the number of clusters, $R$, and the dimensions of the partially common subspaces, $L_r$, are known and fixed. Then, we are ready for investigating the first question.

\begin{proposition}[\hspace{0.01mm}\cite{ding2006orthogonal}]
For a given matrix $\mathbf{X}\in\mathbb{R}^{T\times D}$, consider the matrix factorization $\mathbf{X}=\mathbf{A}\mathbf{B}^T$, where $\mathbf{A}\in\mathbb{R}^{T\times R}_+$ satisfies $\mathbf{A}^T\mathbf{A}=\mathbf{I}_R$ and $\mathbf{B}\in\mathbb{R}^{D\times R}$. Assume that $\mathbf{X}=\hat{\mathbf{A}}\hat{\mathbf{B}}^T$ also holds for a different pair of matrices $\hat{\mathbf{A}}\in\mathbb{R}^{T\times R}_+$, satisfying $\hat{\mathbf{A}}^T\hat{\mathbf{A}}=\mathbf{I}_R$, and $\hat{\mathbf{B}}\in\mathbb{R}^{D\times R}$. Then, there exists a permutation matrix $\mathbf{P}\in\mathbb{R}^{R\times R}$ such that $\hat{\mathbf{A}}=\mathbf{A}\mathbf{P}$ and $\hat{\mathbf{B}}=\mathbf{B}\mathbf{P}^T$.
\end{proposition}


\begin{proposition}
A symmetric BTD of tensor $\mathbfcal{Z}$ in (\ref{TensorModel}), under the constraints: (i) $\mathbf{G}_r^T\mathbf{G}_r=\alpha_r\mathbf{I}_{L_r}$, (ii) $\mathbf{C}^T\mathbf{C}=\mathbf{I}_{R}$, and (iii) $\mathbf{C}\geq \mathbf{0}$, is essentially unique up to permutations of the rank one terms and multiplications of factor bases, $\mathbf{G}_r$, with scaled unitary matrices $\mathbf{Q}_r\in\mathbb{R}^{L_r\times L_r}$.
\end{proposition}

\begin{proof}
    Consider a symmetric BTD of tensor $\mathbfcal{Z}=\left(\mathbf{G}, \mathbf{G}, \mathbf{C}\right)$ that satisfies constraints (i), (ii), and (iii). Let $\mathbf{Z}^{(3)}=\mathbf{C}\mathbf{W}^T$, for $\mathbf{W}(:,r) = \text{vec}\left(\mathbf{G}_r\mathbf{G}_r^T\right)$, be the matricization of $\mathbfcal{Z}$ with respect to its third mode. For any other symmetric BTD of $\mathbfcal{Z}=\left(\hat{\mathbf{G}}, \hat{\mathbf{G}}, \hat{\mathbf{C}}\right)$, we can also get that $\mathbf{Z}^{(3)}=\hat{\mathbf{C}}\hat{\mathbf{W}}^T$, for $\hat{\mathbf{W}}(:,r) = \text{vec}\left(\hat{\mathbf{G}}_r\hat{\mathbf{G}}_r^T\right)$. By Proposition 1, if $\hat{\mathbf{C}}$ satisfies constraints (ii) and (iii), then there exists a permutation matrix, $\mathbf{P}\in\mathbb{R}^{R\times R}$, such that $\hat{\mathbf{C}}=\mathbf{C}\mathbf{P}$ and $\hat{\mathbf{W}}=\mathbf{W}\mathbf{P}$. Therefore, the symmetric BTD of $\mathbfcal{Z}$ is essentially unique up to permutations of its rank one terms. 

    Moreover, we can see that there exists a one-to-one correspondence between the columns of $\hat{\mathbf{W}}$ and $\mathbf{W}$. Let $i,j\in\left[R\right]$ such that $\hat{\mathbf{W}}[:,i]=\mathbf{W}[:,j]$. Then, we get that $\hat{\mathbf{G}}_i\hat{\mathbf{G}}_i^T=\mathbf{G}_j\mathbf{G}_j^T$. Because of constraint (i), we see that $$\alpha_i^2\mathbf{I}_R=\hat{\mathbf{G}}_i^T\mathbf{G}_j\mathbf{G}_j^T\hat{\mathbf{G}}_i=\mathbf{Q}_{ij}\mathbf{Q}^T_{ij}.$$ Hence, matrix $\mathbf{Q}_{ij}=\hat{\mathbf{G}}_i^T\mathbf{G}_j$ is a unitary matrix multiplied by a scalar and $\mathbf{G}_j=\alpha_j^{-1}\hat{\mathbf{G}}_i\mathbf{Q}_{ij}$, which concludes the proof. 
\end{proof}

The second question focuses on tensor $\mathbfcal{P}$ and involves specifying conditions under which a decomposition of $\mathbfcal{P}$ is essentially unique. One may notice that relation (\ref{TensorModel}) hints at a symmetric BTD of $\mathbfcal{P}$. This becomes clear after letting $\bar{\mathbf{C}}=\left[\mathbf{C}, ~\mathbf{I}_K\right]\in\mathbb{R}^{K\times R+K}$, $\bar{\mathbf{G}}_r=\mathbf{G}_r\in\mathbb{R}^{N\times L_r}$ for $r\in\left[R\right]$, and $\bar{\mathbf{G}}_{R+k}=\mathbf{H}_{k}\in\mathbb{R}^{N\times L_{R+k}}$ for $k\in\left[K\right]$, where $L_{R+k}=M_{k}-L_{r_k}$ and $r_k$ corresponds to the index of the nonzero element of $\mathbf{C}(k,:)$. Hence, relation (\ref{TensorModel}) can be rewritten as
\begin{equation}
\mathbfcal{P}(:,:, k) = \sum_{r=1}^{R+K} \bar{\mathbf{c}}_r(k)\bar{\mathbf{G}}_r\bar{\mathbf{G}}_r^T,~~~\forall k\in\left[K\right].
\label{TensorModelII}
\end{equation}
Next, we present sufficient conditions under which a symmetric BTD of $\mathbfcal{P}$ is essentially unique based on Theorem 4.7 in \cite{de2008decompositions}. 

\begin{definition}
\label{def:kruskal_rank_matrix}
The \emph{Kruskal rank} of a matrix $\mathbf{A}$, denoted by $k_{\mathbf{A}}$, is the maximal integer $k$ such that \emph{any} subset of $k$ columns of $\mathbf{A}$ is linearly independent.
\end{definition}

\begin{definition}
\label{def:kruskal_rank_partitioned}
The \emph{Kruskal rank} of a (not necessarily uniformly) partitioned matrix $\mathbf{A}=\left[\mathbf{A}_1,\cdots,\mathbf{A}_T\right]$, denoted by $k'_{\mathbf{A}}$, is defined as the maximal integer $r$ such that \emph{any} set of $r$ submatrices of $\mathbf{A}$ yields a set of linearly independent columns.
\end{definition}

\begin{theorem}
Let $(\bar{\mathbf{G}}, \bar{\mathbf{G}}, \bar{\mathbf{C}})$ represent a symmetric BTD of $\mathbfcal{P}$ in generic rank-($L_r,L_r,1)$ terms, $1\leq r\leq R+K$, satisfying
(a) $\bar{\mathbf{G}}_r^T\bar{\mathbf{G}}_r=\mathbf{I}_{L_r}$ and
    (b) $\bar{\mathbf{C}}\in \left\{0,1\right\}^{K\times R}$ where $\bar{\mathbf{C}}^T\bar{\mathbf{C}}=\text{diag}(\boldsymbol{\alpha})$ for a vector $\boldsymbol{\alpha}\in\mathbb{R}_+^{R+K}$.
Suppose that we have an alternative decomposition of $\mathbfcal{P}$, satisfying (a) and (b), represented by $(\tilde{\mathbf{G}}, \tilde{\mathbf{G}}, \tilde{\mathbf{C}})$, with $k'_{\tilde{\mathbf{G}}}$ maximal under the given dimensionality constraints. If the conditions
\begin{itemize}
    \item[(i)] $N^2\geq \sum_{r=1}^{R} \left(1+\boldsymbol{\alpha}(r)\right)L_{r}^2 + \sum_{k=1}^{K} M_k\left(M_k-2L_{r_k}\right)$
    \item[(ii)] $k'_{\bar{\mathbf{G}}}\geq R+K+1-\frac{1}{2}\min_{r\in\left[R\right]}\boldsymbol{\alpha}(r)$
\end{itemize}
hold, then $(\bar{\mathbf{G}}, \bar{\mathbf{G}}, \bar{\mathbf{C}})$ and $(\tilde{\mathbf{G}}, \tilde{\mathbf{G}}, \tilde{\mathbf{C}})$ are essentially equivalent. 
\end{theorem}

\begin{proof}
The first condition of Theorem 4.7 in \cite{de2008decompositions} translates, in our case, to the following condition
\begin{equation*}
\begin{split}
    N^2&\geq\sum_{r=1}^{R+K}L_r^2\\
    &=\sum_{r=1}^{R} L_{r}^2 + \sum_{k=1}^{K} (M_k-L_{r_k})^2\\
    &=\sum_{r=1}^{R} \left(1+\boldsymbol{\alpha}(r)\right)L_{r}^2 + \sum_{k=1}^{K} M_k\left(M_k-2L_{r_k}\right),
    \end{split}
\end{equation*}
where we used that vector $\boldsymbol{\alpha}$ holds the number of members of each cluster in its first $R$ elements.

The second condition of Theorem 4.7 in \cite{de2008decompositions} translates, in our case, to
\begin{equation*}
\begin{split}
2k'_{\bar{\mathbf{G}}} + k_{\bar{\mathbf{C}}} \geq 2(R+K) + 2.
 \end{split}
\end{equation*}
For the Kruskal's rank of $\bar{\mathbf{C}}$ we have that expressing the $r$-th column of $\mathbf{C}$ requires the linear combination of exactly $\boldsymbol{\alpha}(r)$ columns of the identity matrix. As a result, $k_{\bar{\mathbf{C}}}\leq\min_{r\in\left[R\right]}\boldsymbol{\alpha}(r)$, which leads us to 
$k'_{\bar{\mathbf{G}}}\geq R+K+1-\frac{1}{2}\min_{r\in\left[R\right]}\boldsymbol{\alpha}(r)$.

\end{proof}

The first condition poses some constraints on the dimensions of the partially common subspaces, the dimensions of matrices $\mathbf{X}_k$, and the sizes of the different clusters. For example, in the case where $L_r=L$ and $M_k=M$ for all $r$ and $k$, we get that $N^2\geq RL^2+K(M-L)^2$. For fixed $N, L$, and $M$, we see that this inequality sets a trade-off between the number of clusters and the number of samples for essential uniqueness. The second condition requires, for the extreme case where a cluster has only two members, to have $k'_{\bar{\mathbf{G}}}=R+K$. This implies that matrix $\bar{\mathbf{G}}=\left[\mathbf{G}_{1},\ldots, \mathbf{G}_{R}, \mathbf{H}_{1},\ldots, \mathbf{H}_{K}\right]$ has to be full column rank. In other words, the different subspaces spanned by $\mathbf{G}_r$ or $\mathbf{H}_k$ have to be quite diverse in order to facilitate the unique identification of a cluster of two members. In contrast, an extreme scenario where each pair of matrices $\mathbf{X}_k$ shares a partially common subspace that differs across pairs would inevitably compromise the uniqueness of a symmetric BTD of $\mathbfcal{P}$.


In general, as the minimum number of members increases, the value of $k'_{\bar{\mathbf{G}}}$ that is required by Theorem 1 decreases but slowly. For example, in the case of two clusters where each has $K/2$ members, we get that $k'_{\bar{\mathbf{G}}} \geq 3 +\frac{3}{4}K$, when its maximum value can be $K+2$. Hence, for $K=4$ the requirement becomes $6$ which is the maximum possible, while for $K=8$ the requirement drops to $9$ from the maximum possible of $10$.


\section{Methods}

In this section, we propose an alternating least squares algorithm for solving subproblems of the form presented in (\ref{FitCri}) and (\ref{eq:aug_lagrangian}). In the next subsections, we overview the proposed clustering algorithms and we describe in detail all the update rules that take place during its execution. We also provide a scalable algorithm that performs the corresponding updates using only bases $\mathbf{U}_k$, instead of using the corresponding projection matrices and instantiating tensor $\mathbfcal{P}$, which could be prohibitive for high dimensional matrices. At last, we consider the problems of estimating the number of clusters, the dimensions of the partially ``common" subspaces, and we provide some suggestions on addressing those problems.

\subsection{Proposed Clustering Algorithms}

The proposed algorithm attempts to solve each of the subproblems (for a given value of $\rho_t$ in (\ref{FitCri}) or $\rho_t$ and $\boldsymbol{\Lambda}_t$ in (\ref{eq:aug_lagrangian}) and (\ref{eq:aug_lagrangian_II})) sequentially, while the solution of each subproblem is considered as initialization for the next one. For each subproblem, the presented iterative algorithm performs updates on each factor $\widehat{\mathbf{G}}_r$ sequentially, while factor $\widehat{\mathbf{C}}$ is updated at the end of each round. These factor update rounds are repeated until convergence is reached.

Several conditions can be used as stopping criteria. In our experiments, we check whether the relative change of the cost function (when its evaluation is affordable) or the factor matrices over consecutive iterations does not exceed a predefined small threshold. As for the generation of the sequence $\rho_t$, in our experiments we use the rule $\rho_t = \alpha \rho_{t-1}$, where $\alpha$ takes a value larger than one (typically from $\left[1.01,1.5\right]$), while we assign to $\rho_1$ a small positive value (typically from $\left[10^{-6},1\right]$). For the case of the augmented Lagrangian methods in (\ref{eq:aug_lagrangian}) and (\ref{eq:aug_lagrangian_II}), the dual ascent updates are given by $\boldsymbol{\Lambda}_{t+1} = \boldsymbol{\Lambda}_t + \rho_t\left(\hat{\mathbf{C}}^T\widehat{\mathbf{C}}\right)\ast\mathbf{Q}$ and $\boldsymbol{\Lambda}_{t+1} = \boldsymbol{\Lambda}_t + \rho_t\Psi\left(\widehat{\mathbf{C}}\right)\ast\mathbf{Q}$, respectively, with $\boldsymbol{\Lambda}_0=\mathbf{0}$.
Next, we present the proposed update rules for the factors $\widehat{\mathbf{G}}_r$ and $\widehat{\mathbf{C}}$. 

\subsubsection{Updating factors $\widehat{\mathbf{G}}_r$ for the problems in (\ref{FitCri}) and (\ref{eq:aug_lagrangian})}

There are several options for updating the factors $\widehat{\mathbf{G}}_r$. A straightforward approach is to adopt projected gradient descent schemes that constrain the factors $\widehat{\mathbf{G}}_r$ to lie on the Stiefel manifold, i.e., $\widehat{\mathbf{G}}_r^T \widehat{\mathbf{G}}_r = \mathbf{I}_{L_r}$. However, in our experience, such approaches are quite sensitive and tend to get easily stuck in poor local minima. Next, we present the approach we propose in \cite{karakasis2023clustering}, which relies on a deflation-based scheme.

We begin by considering the matricization of tensor $\mathbfcal{P}$ with respect to to the third mode, $\mathbf{P}^{(3)}\in\mathbb{R}^{K\times N^2}$, defined as \cite{de2008decompositions} 
\begin{equation}
\mathbf{P}^{(3)}(k,:):=\rm{vec}\left(\mathbf{U}_k\mathbf{U}_k^T\right)^T.
\end{equation}
For a given set of factors $\widehat{\mathbf{G}}_1$, $\ldots$, $\widehat{\mathbf{G}}_R$, and $\hat{\mathbf{C}}$, assume that we are interested in updating factor $\widehat{\mathbf{G}}_{r'}$, for some $r'$. After letting 
$$\mathbf{D}^{(3)}_{r'}=\mathbf{P}^{(3)}-\sum_{r=1,r\neq r'}^R \widehat{\mathbf{c}}_r\rm{vec}\left(\hat{\mathbf{G}}_r\hat{\mathbf{G}}_r^T\right)^T,$$
we can get that the problem of interest is 
\begin{equation}
    \begin{split}
    \min_{\widehat{\mathbf{G}}_{r'}} & \left\|\mathbf{D}^{(3)}_{r'}-\hat{\mathbf{c}}_{r'}\rm{vec}\left(\widehat{\mathbf{G}}_{r'}\widehat{\mathbf{G}}_{r'}^T\right)^T\right\|_F^2\\
\text{s.t. }& ~~~~~~~~~~\widehat{\mathbf{G}}_{r'}^T\widehat{\mathbf{G}}_{r'}=\mathbf{I}_{L_{r'}}.
\label{LS}
   \end{split}
\end{equation}

By virtue of the Optimal Scaling Lemma (cf. \cite{bro1998least}, see also \cite{ sidiropoulos2017tensor}), it can be shown that problem (\ref{LS}) is equivalent to
\begin{equation}
    \begin{split}
    \min_{\widehat{\mathbf{G}}_{r'}} & \left\|\widehat{\mathbf{W}}_{r'}-\widehat{\mathbf{G}}_{r'}\widehat{\mathbf{G}}_{r'}^T\right\|_F^2~~~
\text{s.t. }~~~\widehat{\mathbf{G}}_{r'}^T\widehat{\mathbf{G}}_{r'}=\mathbf{I}_{L_{r'}},
\label{LS2}
   \end{split}
\end{equation}
where 
\begin{equation}
\hat{\mathbf{W}}_{r'}=\sum_{k=1}^K\frac{\hat{\mathbf{c}}_{r'}(k)}{\left\|\hat{\mathbf{c}}_{r'}\right\|^2_2}\mathbf{U}_k\mathbf{U}_k^T-\sum_{r=1,r\neq r'}^R \frac{\hat{\mathbf{c}}_{r'}^T\hat{\mathbf{c}}_r}{\left\|\hat{\mathbf{c}}_{r'}\right\|^2_2}\hat{\mathbf{G}}_r\hat{\mathbf{G}}_r^T.
\label{W}
\end{equation}
Notice that, by definition, matrix $\hat{\mathbf{W}}_{r'}$ is always real and symmetric. Hence, since problem (\ref{LS2}) is equivalent to 
\begin{equation}
    \max_{\widehat{\mathbf{G}}_{r'}} ~~~\text{trace}\left(\hat{\mathbf{G}}_{r'}^T\hat{\mathbf{W}}_{r'}\hat{\mathbf{G}}_{r'}\right)~~
\text{s.t. }~~ \hat{\mathbf{G}}_{r'}^T\hat{\mathbf{G}}_{r'}=\mathbf{I}_{L_{r'}},
   \raisetag{20mm}
   \label{UpdAr}
\end{equation}
an optimal solution of (\ref{LS}) is given by the top $L_{r'}$ eigenvectors of matrix $\hat{\mathbf{W}}_{r'}$ \cite{golub2013matrix}. 

Regarding the computational complexity for updating each factor $\hat{\mathbf{G}}_{r'}$, let $M_{\text{max}}=\max_k M_k$ and $L_{\text{max}}=\max_r L_r$. Then, we have that the computation of the corresponding matrix $\hat{\mathbf{W}}_{r'}$ requires $\mathcal{O}\left(N^2\left(KM_{\text{max}}+RL_{\text{max}}\right)+KR\right)$ flops, while computing the eigenvalue decomposition of matrix $\hat{\mathbf{W}}_{r'}$ requires $\mathcal{O}\left(N^3\right)$ flops. 
As a result, the overall computational complexity for updating each factor $\hat{\mathbf{G}}_{r}$ is $$\mathcal{O}\left(N^3+N^2\left(K\max_k M_k+RL_{\text{max}}\right)+KR\right).$$ Although the resulting complexity is linear in the number of views, $K$, and the number of clusters, $R$, we can see that forming matrix $\hat{\mathbf{W}}_{r}$ and computing its eigenvalue decomposition induces the main computational bottleneck of this step. Moreover, for very high-dimensional data, even forming matrix $\hat{\mathbf{W}}_{r}$ becomes challenging. Next, we present a scalable alternative where the formation of matrix $\hat{\mathbf{W}}_{r}$ is avoided by adopting a light-weight iterative approach.

\paragraph{Scalable approach for updating factors $\hat{\mathbf{G}}_r$}

The power method \cite{golub2013matrix} is a classical iterative method for computing the principal eigenvector of a symmetric matrix. When more than one principal eigenvectors are desired, one may employ its extension, known as Orthogonal Iteration \cite{golub2013matrix}. Orthogonal Iteration produces a sequence of orthonormal frames that, under mild conditions, converges to the desired top eigenvectors of a symmetric matrix. Each frame of this sequence is obtained by computing the product of the matrix with the previous frame and then keeping the Q term from its QR decomposition. 

Interestingly, one may notice that in our case, the product $\mathbf{W}_{r'}\mathbf{A}$, for a matrix $\mathbf{A}\in\mathbb{R}^{N\times L_{r'}}$, can be calculated very efficiently, without instantiating $\mathbf{W}_{r'}$, just by exploiting its structure appearing in (\ref{W}). As a result, the complexity of computing each term of the sequence of the Orthogonal Iteration is $\mathcal{O}\left(NL_{r'}\left(KM_{max}+RL_{max}\right)\right)$. We can observe that adopting the Orthogonal Iteration allows us to obtain a scalable algorithm for updating each basis $\hat{\mathbf{G}}_{r'}$, with a complexity that is linear in the number of views $K$, the number of clusters $R$, the maximum number of columns of a view $M_{max}$, and quadratic in the dimensions of the partially common subspaces $L_r$.

\subsubsection{Updating factors $\hat{\mathbf{G}}_r$ for the problem in (\ref{eq:aug_lagrangian_II})}

The factors $\hat{\mathbf{G}}_r$ are unconstrained in the problem formulation of (\ref{eq:aug_lagrangian_II}). As a result, we propose the consideration of first order methods for updating them. Computing the gradient w.r.t. to factor $\hat{\mathbf{G}}_{r'}$ can be done very efficiently as there is structure that can be exploited. Specifically, for a given matrix $\widehat{\mathbf{C}}$, it can be shown that the derivative of $f_{\mathbfcal{P}}$ w.r.t. $\hat{\mathbf{G}}_{r'}$ is given by 
\begin{equation*}
    \frac{\partial f_{\mathbfcal{P}}}{\partial \hat{\mathbf{G}}_{r'}}=2\left[\sum_{s=1}^R\widehat{\mathbf{c}}_s^T\widehat{\mathbf{c}}_{r'}\hat{\mathbf{G}}_s\hat{\mathbf{G}}_s^T-\sum_{k=1}^K\widehat{\mathbf{C}}\left(k,r'\right)\mathbf{U}_k\mathbf{U}_k^T\right]\hat{\mathbf{G}}_{r'},
\end{equation*}
which can be computed with complexity $\mathcal{O}\left(NL_{r'}(KM_{max}+RL_{max})+KR\right)$, and is linear in the number of views $K$, the number of clusters $R$, the maximum number of columns of a view $M_{max}$, and quadratic in the dimensions of the partially common subspaces $L_r$.

\subsubsection{Updating factor $\mathbf{C}$}

In \cite{karakasis2023clustering}, we show that updating the factor $\mathbf{C}$, in the problem formulation of (\ref{FitCri}), can be efficiently performed using the optimal first-order Nesterov-type algorithm for $L$-smooth, $\mu$-strongly convex Nonnegative Matrix Least Squares (NMLS) problems, as proposed in \cite{liavas2017nesterov}. Here, we consider a more general framework for handling the updates of $\mathbf{C}$ for all formulations in (\ref{FitCri}), (\ref{eq:aug_lagrangian}), and (\ref{eq:aug_lagrangian_II}).

Given a set of factors $\widehat{\mathbf{G}}_r$, let $$\widehat{\mathbf{B}} := \left[\text{vec}\left(\widehat{\mathbf{G}}_1\widehat{\mathbf{G}}_1^T\right)\hspace{2mm}\ldots\hspace{2mm}\text{vec}\left(\widehat{\mathbf{G}}_R\hat{\mathbf{G}}_R^T\right)\right]$$
and $$f_{\mathbfcal{P}|\widehat{\mathbf{G}}}(\widehat{\mathbf{C}}):=\frac{1}{2}\left\|\mathbf{P}^{(3)}
-\widehat{\mathbf{C}}\widehat{\mathbf{B}}^T\right\|_F^2.$$ Then, the problem of updating factor $\mathbf{C}$ is the regularized NMLS problem
\begin{equation}
\min_{\widehat{\mathbf{C}}} ~~~~f_{\mathbfcal{P}|\widehat{\mathbf{G}}}(\widehat{\mathbf{C}})~+~\phi(\widehat{\mathbf{C}})~~~\text{s.t.}~~~~~\widehat{\mathbf{C}}\geq \mathbf{0},
\label{UpdC}
\end{equation}
where the function $\phi$ varies depending on the problem formulation under consideration, i.e., (\ref{FitCri}), (\ref{eq:aug_lagrangian}), or (\ref{eq:aug_lagrangian_II}). To tackle this optimization problem, we propose a more abstract approach than the one in \cite{karakasis2023clustering}. Once again, the central idea is to employ optimal first-order methods. 

While computing the gradient of the function $\phi$ analytically can be quite complex, automatic differentiation tools have become increasingly popular for computing gradients of such functions. On the other hand, the gradient of $f_{\mathbfcal{P}|\widehat{\mathbf{G}}}$ can be computed in a straightforward manner. 
Moreover, as we show next, this computation can be performed very efficiently due to exploitable structure that may be overlooked by automatic differentiation methods. Therefore, we propose using automatic differentiation for the gradient of $\phi$, while we derive an efficient custom computation for the gradient of $f_{\mathbfcal{P}|\widehat{\mathbf{G}}}$.

The gradient of $f_{\mathbfcal{P}|\widehat{\mathbf{G}}}$ w.r.t. $\widehat{\mathbf{C}}$ is given by
\begin{equation}
\nabla f_{\mathbfcal{P}|\widehat{\mathbf{G}}}(\widehat{\mathbf{C}}) = \widehat{\mathbf{C}}\,\widehat{\mathbf{B}}^T\widehat{\mathbf{B}} - \mathbf{P}^{(3)}\widehat{\mathbf{B}}.
\end{equation}
The main computational bottleneck lies in evaluating the terms $\widehat{\mathbf{B}}^T\widehat{\mathbf{B}}$ and $\mathbf{P}^{(3)}\widehat{\mathbf{B}}$, which—although expensive—can be computed once and reused throughout the process of updating $\widehat{\mathbf{C}}$. A naive computation of these terms has complexity $\mathcal{O}\left(N^2KR\right)$. However, once they are available, the complexity of computing the gradient $\nabla f_{\mathbfcal{P}|\widehat{\mathbf{G}}}$ per iteration is reduced to $\mathcal{O}\left(KR^2\right)$.

Importantly, there is exploitable structure that can significantly reduce the cost of computing $\widehat{\mathbf{B}}^T\widehat{\mathbf{B}}$. Specifically, we observe that
\begin{equation}
\begin{split}
[\widehat{\mathbf{B}}^T\widehat{\mathbf{B}}](r_1, r_2) 
&= \mathrm{vec}(\widehat{\mathbf{G}}_{r_1}\widehat{\mathbf{G}}_{r_1}^T)^T \mathrm{vec}(\widehat{\mathbf{G}}_{r_2}\widehat{\mathbf{G}}_{r_2}^T) \\
&= \mathrm{trace}(\widehat{\mathbf{G}}_{r_1}^T \widehat{\mathbf{G}}_{r_2} \widehat{\mathbf{G}}_{r_2}^T \widehat{\mathbf{G}}_{r_1}) \\
&= \|\widehat{\mathbf{G}}_{r_1}^T \widehat{\mathbf{G}}_{r_2}\|_F^2.
\end{split}
\end{equation}
Hence, forming the matrix $\widehat{\mathbf{B}}^T\widehat{\mathbf{B}}$ can be accomplished with complexity $\mathcal{O}\left(NR^2 L_{max}^2\right)$, which is linear in $N$. The same holds for the term $\mathbf{P}^{(3)}\widehat{\mathbf{B}}$, which can be computed, without instantiating matrices $\mathbf{P}^{(3)}$ and $\widehat{\mathbf{B}}$, as follows
\begin{equation}
\begin{split}
    [\mathbf{P}^{(3)}\widehat{\mathbf{B}}](k,r) &= \rm{vec}\left(\mathbf{U}_k\mathbf{U}_k^T\right)^T\rm{vec}(\widehat{\mathbf{G}}_{r}\hat{\mathbf{G}}_{r}^T)\\
&=\rm{trace}(\mathbf{U}_k^T\widehat{\mathbf{G}}_{r}\widehat{\mathbf{G}}_{r}^T\mathbf{U}_k)\\
 &=\|\mathbf{U}_k^T\widehat{\mathbf{G}}_{r}\|_F^2,
    \end{split}
\end{equation}
with complexity $\mathcal{O}\left(NKRM_{max}L_{max}\right)$.

Therefore, forming the two necessary terms for computing the gradient $\nabla f_{\mathbfcal{P}|\widehat{\mathbf{G}}}$ entails an overall complexity of $$\mathcal{O}\left(NKRM_{max}L_{max}\right),$$ which is linear in every participating factor. We remind that their computation is required only once at the start of the NMLS algorithm. Given these terms, the complexity of each subsequent gradient computation of $f_{\mathbfcal{P}|\widehat{\mathbf{G}}}$ can be reduced to $\mathcal{O}\left(KR^2\right)$.

\subsection{Determining the number of clusters}

The focus of this subsection is to device a metric for assessing how appropriate is a specific choice of $R$. Evaluating $f_{\mathbfcal{P}}$, defined in (\ref{ProbBasic}), seems a natural choice for this task. Also, it does not necessarily rely on instantiating tensor $\mathbfcal{P}$, as it can be done in an efficient way even for high dimensional matrices (large values of $N$). More specifically, let $(\widehat{\mathbf{G}}^*,\widehat{\mathbf{G}}^*, \widehat{\mathbf{C}}^*)$ be a final solution of the sequence of subproblems in (\ref{FitCri}), (\ref{eq:aug_lagrangian}), or (\ref{eq:aug_lagrangian_II}). We can recover the assignment of each matrix $\mathbf{X}_k$ to one of the $R$ clusters by considering the index of the largest value of the corresponding row of matrix $\widehat{\mathbf{C}}^*$. In the sequel, we denote it as $r_k$. Then, it can be shown that the cost function $f_{\mathbfcal{P}}(\widehat{\mathbf{G}}^*, ~\widehat{\mathbf{C}}^*)$ will be equal to
\begin{equation}
    \sum_{k=1}^KM_k + L_{r_k} - 2\|\mathbf{U}_k^T\widehat{\mathbf{G}}_{r_k}^*\|_F^2.
\end{equation}
In the case where $M_k\neq L_{r_k}$, the term above cannot reach zero, while the different values of $M_k$ makes the use of this metric less intuitive and interpretable. 

Instead, we propose using the following metric that relies on
\begin{equation}
    \phi_k:= 1-\frac{1}{L_{r_k}}\left\|\mathbf{U}_k^T\mathbf{G}_{r_k}\right\|_F^2=\frac{1}{L_{r_k}}\sum_{i=1}^{L_{r_k}}\sin^2\left(\theta_{k,i}\right),
    \label{dist_cl}
\end{equation}
where $\theta_{k,i}$ denotes the $i$-th principal angle between $\text{col}\left(\mathbf{U}_k\right)$ and the assigned subspace $\text{col}\left(\mathbf{G}_{r_k}\right)$. As a result, adding up the corresponding terms of (\ref{dist_cl}) for all matrices $\mathbf{X}_k$, i.e. $\bar{\phi}:=\frac{1}{K}\sum_k \phi_k$, results in an assessment of how successful the overall assignment to the different clusters is. More specifically, the lower the value of $\bar{\phi}$, the better the assignment. Moreover, the proposed metric takes values in the (0, 1) interval, which makes it intuitive, practical, and interpretable. Ideally, the $\bar{\phi}$ metric would decrease for increasing numbers of clusters. On the other hand, if matrices $\mathbf{X}_k$ can be grouped well with $R^*$ clusters, increasing the number of clusters beyond $R^*$ would not result in a significant drop of $\bar{\phi}$. Based on this rationale, we propose solving the sequence of subproblems in (\ref{FitCri}), (\ref{eq:aug_lagrangian}), or (\ref{eq:aug_lagrangian_II}), for increasing values of $R$ and evaluating $\bar{\phi}$ until no significant drop in $\bar{\phi}$ is observed.

\subsection{On the dimensions of the partially common subspaces}

The dimensions of the partially common subspaces, $L_r$, affect the term $\bar{\phi}$. However, unlike what happens when we increase $R$, increasing the values of $L_r$ can result in an increment of $\bar{\phi}$. To see that, let matrix $\mathbf{X}_k$ follow the generative model in (\ref{GenModel}) and participate in the $r$-th  partially common subspace of dimension $L_{r}^*$. Then, in the absence of noise, the evaluation of $\phi_k$ in (\ref{dist_cl}) would give $0$. However, over-modelling its dimension, using $L_{r_k}'>L_{r_k}^*$, will result in a strictly positive $\phi_k$. On the other hand, under-modelling its dimension, using $L_{r_k}'<L_{r_k}^*$, will result in a $\phi_k$ equal to $0$. We conclude that, in the noiseless case and for a fixed $R$, examining high values of $L_r$ that keep $\bar{\phi}$ small is a meaningful criterion for determining the appropriate dimensions of the partially common subspaces.

One may argue that there are at least two main disadvantages in the approach detailed above. First, it requires solving a combinatorial problem, i.e., finding the optimal combination of maximal integer values for the dimensions $L_r$ that minimize $\bar{\phi}$. Evaluating each specific guess involves solving another sequence of problems, the ones appearing in (\ref{FitCri}), (\ref{eq:aug_lagrangian}), or (\ref{eq:aug_lagrangian_II}). Consequently, a brute-force approach would be computationally demanding and therefore impractical. The second disadvantage is that it relies on an overly idealistic assumption - the absence of noise. In the presence of noise, more effective approaches are needed for determining the dimensions $L_r$.
In what follows, we focus on the problem of estimating the dimensions $L_r^*$, together with the corresponding bases of the partially common subspaces $\mathbf{G}_r\in\mathbb{R}^{N\times L_r^*}$, starting from a given solution of (\ref{FitCri}), (\ref{eq:aug_lagrangian}), or (\ref{eq:aug_lagrangian_II}), that may have been obtained using dimensions other than $L_r^*$.

The main idea is the following. Given a solution of $(\ref{FitCri})$, $(\ref{eq:aug_lagrangian})$, or (\ref{eq:aug_lagrangian_II}), we have an assignment of each matrix $\mathbf{X}_k$ to one of the $R$ clusters. Hence, the problem boils down to estimating the dimension of each partially common subspace independently. For the $r'$-th cluster, let $\mathbb{I}_{r'}$ be the set of indices of matrices $\mathbf{X}_k$ that participate in it. Then, for the assignment under consideration and in the absence of noise, the conditionally optimal estimate of dimension $L_{r'}$, $L_{r'}^*$, is the maximal integer value for which the optimal cost value of the MAXVAR problem,
\begin{equation}
\begin{split}
   \min_{\mathbf{G}_{r'}, \left\{\mathbf{Q}_k\right\}_{k\in\mathbb{I}_{r'}}} &\sum_{k\in \mathbb{I}_{r'}}\left\|\mathbf{X}_k\mathbf{Q}_k-\mathbf{G}_{r'}\right\|_F^2\\
   \text{s.t.}~~~~ & ~~~~~\mathbf{G}_{r'}^T\mathbf{G}_{r'} = \mathbf{I}_{L_{r'}},
\end{split}
\end{equation}
is equal to zero. The MAXVAR problem formulation can be written only in terms of the basis, $\mathbf{G}_{r'}$, which leads to the problem of finding the top eigenvectors,
\begin{equation}
\begin{split}
    \max_{\mathbf{G}_{r'}\in\mathbb{R}^{N\times L_{r'}}} & \text{trace}\left(\mathbf{G}_{r'}^T\mathbf{T}_{r'}\mathbf{G}_{r'}\right)~~\text{s.t.}~~\mathbf{G}_{r'}^T\mathbf{G}_{r'} =  \mathbf{I}_{L_{r'}},
\end{split}
\end{equation}
of matrix $$\mathbf{T}_{r'}=\sum_{k\in\mathbb{I}_{r'}}\mathbf{X}_k\left(\mathbf{X}_k^T\mathbf{X}_k\right)^{-1}\mathbf{X}_k^T=\sum_{k\in\mathbb{I}_{r'}}\mathbf{P}_{\text{col}\left(\mathbf{X}_k\right)},$$ as defined in (\ref{Ptensor}). Therefore, estimating the dimension $L_{r'}^*$ is equivalent to estimating the effective rank of matrix $\mathbf{T}_{r'}$. 

Over the years, various criteria have been introduced to estimate the dimension of the signal subspace when noise is present. These include widely used methods such as the Akaike Information Criterion (AIC) \cite{akaike1974new} and the Minimum Description Length (MDL) \cite{rissanen1978modeling, schwarz1978estimating}—for a detailed discussion, see \cite{liavas2002behavior}—as well as more advanced approaches like the ones that rely on ``maximally stable" decompositions into signal and noise subspaces \cite{liavas1999blind} or on obtaining a pair of estimates for $\mathbf{G}_{r'}$ and then measuring the gap between them \cite{karakasis2022multisubject}. What is common to all these criteria is that they rely on having access to the eigendecomposition of matrix $\mathbf{T}_{r'}$. 

For very high dimensions, $N$, computing the eigenvalue decomposition becomes intractable. An alternative that can be applied here when the number of matrices, $\mathbf{X}_k$, participating in $\mathbb{I}_{r'}$ is relatively small (i.e. $N>(\sum_{k\in \mathbb{I}_{r'}}M_k)^2$) relies on the observation presented in \cite{peppas2023harnessing, 10871926}. More specifically, after recalling (\ref{Ptensor}) and letting $$\bar{\mathbf{U}}_{\mathbb{I}_{r'}}=\left[\ldots ~ ~\mathbf{U}_{k\in\mathbb{I}_{r'}}~\ldots\right]\in\mathbb{R}^{N\times \sum_{k\in \mathbb{I}_{r'}}M_k},$$ it can be observed that an analytical decomposition of $\mathbf{T}_{r'}$ can be expressed in terms of matrices $\mathbf{U}_k$, i.e. $$\mathbf{T}_{r'}=\bar{\mathbf{U}}_{\mathbb{I}_{r'}}\bar{\mathbf{U}}_{\mathbb{I}_{r'}}^T.$$ As a result, when $N>(\sum_{k\in \mathbb{I}_{r'}}M_k)^2$, one may consider the singular value decomposition (SVD) of matrix $\bar{\mathbf{U}}_{\mathbb{I}_{r'}}$ for recovering the eigenvalues and eigenvectors of $\mathbf{T}_{r'}$ with a complexity that is linear in $N$, $\mathcal{O}\left(N(\sum_{k\in \mathbb{I}_{r'}}M_k)^2\right)$. Given these, the aforementioned criteria can be used for determining the dimension of the $r$-th partially common subspace, $L_{r'}^*$, as well as its $L_{r'}^*$-th dimensional basis.

A summary of the overall proposed process for determining the dimensions of all partially common subspaces, along with a pseudocode algorithm, can be found in the supplementary material.

\section{Application to Hyperspectral Imaging: Pixel Clustering}

Hyperspectral Imaging (HSI) is valuable across multiple  scientific domains, from geological exploration and marine monitoring to military reconnaissance, medical imaging, and forensics \cite{ghamisi2017advances, cai2019bs, liu2019deep}. Among the various HSI tasks, supervised pixel classification is the most widely used  \cite{liu2019deep, cai2019discriminative}. However, it relies heavily on large amounts of labeled data, which are often scarce due to the potentially high cost of annotation. To overcome this challenge, considerable research has shifted towards HSI clustering, which seeks to uncover the intrinsic structure of the data and assign labels without requiring manual annotation.

In this setting, each pixel can be viewed as a high-dimensional spectral signature, and subspace clustering methods have gained popularity for their effectiveness in modeling such signatures as lying near low-dimensional subspaces. On the other hand, hyperspectral data are often contaminated with noise due to sensor imperfections and atmospheric interference, which can result in significant corruption or outliers in the spectral measurements \cite{zhai2018laplacian}. Traditional subspace models tend to be sensitive to such noise, limiting their effectiveness in challenging real-world scenarios.
To address these challenges, numerous adaptations of subspace-based HSI clustering have been developed. These include methods that integrate spectral and spatial information \cite{zhang2016spectral}, kernel-based approaches that capture nonlinear structures \cite{patel2014kernel, zhai2017kernel}, and techniques that leverage the inherent graph-structured relationships in hyperspectral data \cite{zhai2018laplacian, cai2020graph}.

Assuming that the spectral signature of a pixel is denoted as $\mathbf{x}\in\mathbb{R}_+^N$, the most common generative model for
$\mathbf{x}$ is the classical Linear Mixing Model (LMM) \cite{adams1986spectral, liangrocapart1998mixed, keshava2000algorithm, keshava2002spectral, bioucas2012hyperspectral}, which represents the observed spectrum as a linear combination of $P$ endmember (pure materials in the scene) spectra with corresponding abundance coefficients. This model can be expressed as
\begin{equation}
    \mathbf{x} = \mathbf{M}\boldsymbol{\alpha}+\mathbf{v},
\end{equation}
where $\mathbf{M}\in\mathbb{R}_+^{N\times P}$ holds the $P$ endmember spectra, $\boldsymbol{\alpha}\in\mathbb{R}_+^P$ is the abundance vector
holding the proportions of each endmember, and $\mathbf{v}\in\mathbb{R}^N$ expresses the additive noise, which accounts for various factors such as sensor noise, atmospheric interference, and other types of distortion.

In most cases, it is reasonable to assume that neighboring pixels in an HSI are composed of similar materials and thus have a high probability of belonging to the same land-cover class \cite{fauvel2008spectral, zhang2016spectral}. In the absence of noise, this suggests that for a collection of $S$ such neighboring pixels, the following relation should hold
\begin{equation}
\text{col}\left(\left[\mathbf{x}_1,\cdots,\mathbf{x}_S\right]\right) \subseteq \text{col}\left(\mathbf{M}\right),
\end{equation}
i.e., the subspace spanned by the spectra of these $S$ pixels is contained within the subspace spanned by the endmember signatures in $\mathbf{M}$.

Now, if we further assume that for each neighborhood the matrix of endmember signatures, $\mathbf{M}$, can be decomposed into two parts—one corresponding to the signatures of the class that the neighborhood belongs to, and another capturing the spectra of materials that appear locally but are unrelated to the associated class—then the generative model for a matrix of $S$ neighboring pixel spectra, $\mathbf{X}_{\text{neigh}} \in \mathbb{R}_+^{N \times S}$, coincides with our proposed model in (\ref{GenModel})
\begin{equation}
    \mathbf{X}_{\text{neigh}} = \left[\mathbf{G}_{r\text{-neigh}},~\mathbf{H}_{\text{neigh}}\right] \mathbf{Q}_{\text{neigh}},
\end{equation}
where $\mathbf{G}_{r\text{-neigh}} \in \mathbb{R}^{N \times L_{r\text{-neigh}}}$ denotes an orthonormal basis spanning the endmember spectra associated with the class of the neighborhood, $\mathbf{H}_{\text{neigh}} \in \mathbb{R}^{N \times (S - L_{r\text{-neigh}})}$ denotes an orthonormal basis spanning locally occurring endmember spectra or noise-related factors, and $\mathbf{Q}_{\text{neigh}} \in \mathbb{R}^{S \times S}$ is the mixing matrix, which is a linear transformation of the abundance coefficients for each pixel in $\mathbf{X}_{\text{neigh}}$.

Representing each data point using the center pixel together with its neighboring pixels is a widely adopted strategy in various HSI spectral–spatial classification methods \cite{cai2019bs, fauvel2012advances, he2017recent}. This is typically done by stacking the PCA-compressed spectral signatures of the neighborhood pixels in a sequential manner \cite{cai2020graph}. However, this raises the important question of how to optimally order the features extracted from the neighborhood of a pixel, as the correspondence between neighboring pixels across different neighborhoods is generally ambiguous. Working with the subspace spanned by the spectra of pixels within a neighborhood not only resolves this ambiguity, but also enhances the contributions of weak endmember signatures that might otherwise be difficult to detect. 

Furthermore, our framework allows for the inclusion of spatial information by considering a more general regularization term in the problem formulation~(\ref{FitCri}):
\begin{equation}
\mathrm{tr}\left(\mathbf{Q}\,\widehat{\mathbf{C}}^T\mathbf{W}_{\mathrm{norm}}\mathbf{W}_{\mathrm{norm}}^T\widehat{\mathbf{C}}\right),
\label{Reg}
\end{equation}
where $\mathbf{W} \in \{0, 1\}^{K \times K}$ is a matrix that encodes the neighborhood relationships between the pixels, capturing spatial dependencies, and $\mathbf{W}_{\mathrm{norm}}$ denotes its column-stochastic version. Although its size is $K \times K$, both matrices $\mathbf{W}$ and $\mathbf{W}_{\mathrm{norm}}$ are highly sparse with $KS$ nonzero entries, where $S$ stands for the considered neighborhood size.
Its construction can be performed only once with complexity $\mathcal{O}\left(K\log(K)\right)$. As a result, the resulting term penalizes assignments to more than one clusters on a neighborhood level. Similarly, the term $\widehat{\mathbf{C}}^T\mathbf{W}_{\mathrm{norm}}\mathbf{W}_{\mathrm{norm}}^T\widehat{\mathbf{C}}$ can be also used in the problem formulation of (\ref{eq:aug_lagrangian}), but also in (\ref{eq:aug_lagrangian_II}) after redefining $\Psi$ as 
\begin{equation}
\boldsymbol{\Psi}_{\text{HSI}}\left(\mathbf{C}\right):=\sqrt{\mathbf{C}_{norm}^T\mathbf{W}_{norm}\mathbf{W}_{norm}^T\mathbf{C}_{norm}+\varepsilon}.
\label{Psi_HSI}
\end{equation}

In our experiments, we have observed that when the penalization parameters become very large, the optimization may prefer to leave some samples unassigned in order to resolve emerging conflicts. Such behavior can range from affecting a few samples to many, which is undesirable. To mitigate this issue, we introduce an additional regularization term, also considered in \cite{zhang2016spectral}, defined as
\begin{equation}
h_{\text{HSI}}\left(\mathbf{C}\right) := \left\|\mathbf{C}_{\mathrm{norm}} - \mathbf{W}_{\mathrm{norm}}^T \mathbf{C}_{\mathrm{norm}} \right\|_F^2.
\label{h_HSI}
\end{equation}
This regularization also relies on the column-wise normalized version of $\mathbf{C}$ to ensure scale invariance, thereby avoiding potential conflicts and numerical instability. Its purpose is to discourage solutions where subsets of rows of $\mathbf{C}$ become zero, by promoting homogeneity among neighboring samples that are likely to belong to the same cluster. 

Using a fixed regularization parameter for \( h_{\text{HSI}} \) is typically not effective, as the penalization for assigning a sample to multiple clusters intensifies with the penalty parameter or the augmented Lagrangian method. To counterbalance the growing pressure against \emph{over-assignment}, it becomes necessary to apply the penalty or augmented Lagrangian formulation to \( h_{\text{HSI}} \) as well. In our experiments, we use the augmented Lagrangian of the form
\begin{equation}
    w_{\mathbfcal{P}|\widehat{\mathbf{G}}}\left(\mathbf{C}\right) + \left<\boldsymbol{\Lambda}^{(h)}_t, \left(\mathbf{I}_K-\mathbf{W}_{\mathrm{norm}}^T\right) \mathbf{C}_{\mathrm{norm}}\right> + \frac{\nu_t}{2}h_{\text{HSI}}\left(\mathbf{C}\right),
    \label{aug_lag_III}
\end{equation}
where $w_{\mathbfcal{P}|\widehat{\mathbf{G}}}$ stands for the combination of $f_{\mathbfcal{P}|\widehat{\mathbf{G}}}$ with any other regularization type emerging from the formulations of (\ref{FitCri}), (\ref{eq:aug_lagrangian}), or (\ref{eq:aug_lagrangian_II}).

\section{Experiments}

In this section, we demonstrate the capabilities of the proposed framework by presenting experimental results. To maintain continuity with the theory developed in the previous section, we first apply our method to real-life hyperspectral images for pixel clustering and compare it against several state-of-the-art methods. The maps generated by all methods across all datasets are provided in the supplementary material of this manuscript.


\subsection{Experiments with Real-Life Hyperspectral Images}

To assess the applicability of the proposed model to real hyperspectral data, we consider three widely used datasets: the Salinas and the Indian Pines scenes acquired by the AVIRIS sensor, and the University of Pavia scene collected by the ROSIS sensor. For computational efficiency, we follow the standard practice in the literature \cite{zhai2017kernel, zhang2016spectral, kong2019hyperspectral, zeng2019spectral, pan2019efficient, cai2020graph} and extract smaller subscenes from each dataset for evaluation. Specifically, the selected subscenes are located within the original images at spatial regions $[30{:}115,\ 24{:}94]$ for Indian Pines and $[150{:}350,\ 100{:}200]$ for Pavia University. Salinas A is typically used in its entirety, as it represents a smaller, self-contained subscene of the full Salinas hyperspectral dataset. The characteristics of each dataset can be found in Table \ref{tab:summary}.

We test our proposed theory by considering, for each pixel, the subspace spanned by its $3 \times 3$ neighborhood and performing Subspace Clustering of Subspaces (SCoS) within our framework. From the proposed framework, we consider the problem formulation of (\ref{eq:aug_lagrangian_II}) together with the variation based on $\boldsymbol{\Psi}_{\text{HSI}}$
in (\ref{Psi_HSI}) and the use of $h_{\text{HSI}}$ in (\ref{h_HSI}) as additional regularization. For the minimization of $h_{\text{HSI}}$ we also consider the augmented Lagrangian method emerging from (\ref{aug_lag_III}).

We compare our approach against a range of baseline and state-of-the-art methods. These include standard clustering techniques such as \text{K-Means} and \text{Spectral Clustering}, two classical subspace clustering algorithms—the \text{Elastic Net Subspace Clustering} (ENSC)~\cite{you2016oracle} and the \text{Block Diagonal Representation-based Subspace Clustering} (BDRSC)~\cite{lu2018subspace}—as well as three methods specifically designed for hyperspectral image (HSI) clustering: $\text{S}^4\text{C}$~\cite{zhang2016spectral}, \text{EGCSC}, and its nonlinear kernel-based extension \text{EKGCSC}~\cite{cai2020graph}.

For ENSC, BDRSC, EGCSC, and EKGCSC, we used the official implementations provided by the original authors.\footnote{See: \url{https://github.com/ChongYou/subspace-clustering}, \url{https://github.com/canyilu/Block-Diagonal-Representation-for-Subspace-Clustering}, \url{https://github.com/AngryCai/GraphConvSC}} The BDRSC method produces two matrices, $\mathbf{B}$ and $\mathbf{Z}$, either of which can be used to construct the affinity matrix~\cite{lu2018subspace}. In our experiments, we use the matrix $\mathbf{Z}$, and we refer to this variant as \text{BDR-Z}. For K-Means and Spectral Clustering, we used the implementations provided by the \texttt{scikit-learn} Python package. Finally, the parameters of all models were manually tuned to maximize performance.

Moreover, to highlight the potential of Subspace Clustering of Subspaces (SCoS), we also propose adaptations of classic Spectral Clustering, Spectral Clustering of Subspaces (Spectral CoS). 
Note that EGCSC and EKGCSC utilize pixel representations based on local neighborhood spectra, where the neighborhood size (denoted as $s_r$ below) is a model parameter. For all methods that use an $s_r\times s_r$ neighborhood of pixels to construct a pixel representation, we indicate the $s_r$ value next to the method's name. This implies that for SCoS-based methods using $3 \times 3$ neighborhoods, $s_r = 3$. 

$\text{S}^4\text{C}$ and all SCoS-based methods require the construction of an adjacency matrix—like the one defined in~(\ref{Reg})—to encode neighboring patches. We denote the size of these neighborhoods as $s_a$ (i.e., $s_a\times s_a$ neighborhoods). Additionally, EGCSC and EKGCSC incorporate graph-based information by constructing an adjacency matrix where each pixel is connected to its $s_c$ most correlated pixels in terms of spectral and not locality. 

In Tables~\ref{tab:indian}, \ref{tab:pavia}, and \ref{tab:salinas}, we report the Adjusted Rand Index (ARI), Normalized Mutual Information (NMI), Overall Accuracy (OA), and Average Per-Class Recall (APR) to evaluate clustering performance. APR reflects the average recall across classes, offering a balanced perspective on performance in imbalanced datasets. We also report APR-C$x$, the recall for class $x$, to highlight class-specific clustering accuracy and reveal trends in cluster purity not captured by aggregate metrics. The top three methods per metric are highlighted in bold. Finally, the maps generated by all methods across all datasets are provided in the supplementary material of this manuscript.

The proposed method shows a significant improvement over the state-of-the-art on the Indian Pines dataset. It achieves an ARI of 0.9219, which is notably higher than the best linear competitor EGCSC with $s_r$ = 9 (ARI = 0.6337), corresponding to a 45.5\% relative increase. Compared to K-Means (ARI = 0.3133), this improvement is even more substantial, exceeding 194\% relative gain. The nonlinear method EKGCSC with $s_r$ = 9 reaches an ARI of 0.6829, which, while better than other linear methods, still falls around 26\% below the proposed method’s ARI. Similar trends hold for NMI and OA, where the proposed method scores 0.9188 and 0.9743, respectively, surpassing EGCSC (NMI = 0.6373, OA = 0.8447) and EKGCSC (NMI = 0.6960, OA = 0.8700) by large margins. The APR metric also reflects this trend, with the proposed method scoring 0.9812 compared to EGCSC’s 0.8868 and EKGCSC’s 0.9068. While EKGCSC’s nonlinear approach provides some performance benefits over linear methods like EGCSC, the proposed method’s linear framework outperforms both, delivering more accurate and robust clustering results on Indian Pines, albeit with a higher computational cost (approximately 39 seconds vs. EGCSC’s 12–13 seconds and EKGCSC’s 14–19 seconds).

On the Pavia dataset, the proposed method again demonstrates strong performance, achieving an ARI of 0.9320, an NMI of 0.9029, an OA of 0.9401, and an APR of 0.8342. Compared to the linear baseline EGCSC with $s_r$ = 9, which achieves an ARI of 0.8129 and OA of 0.8442, the proposed method represents a 14.7\% relative increase in ARI and an 11.3\% relative increase in OA. The nonlinear EKGCSC achieves an ARI of 0.9798 and OA of 0.9586, outperforming the proposed method by about 5.1\% in ARI and 2\% in OA. However, the proposed method still holds a strong competitive edge compared to all linear approaches, outperforming K-Means (ARI = 0.4667) and Spectral Clustering (ARI = 0.3052) by nearly 100\% and over 200\% relative gains, respectively. Notably, the proposed method’s average precision rate (APR) of 0.8342 is close to EKCSC’s 0.8453, showing comparable robustness in cluster purity despite the nonlinear advantage of EKCSC. This high accuracy comes at the cost of higher computational time—approximately 100 seconds for the proposed method compared to 27.6 seconds for EGCSC and 47 seconds for EKGCSC—highlighting a tradeoff between accuracy and computational efficiency.

On the Salinas-A dataset, the proposed method achieves perfect scores across all metrics—ARI, NMI, OA, and APR—all equal to 1.0000, indicating flawless clustering performance. This marks a minor improvement over the linear EGCSC method with $s_r$=9, which attains an ARI of 0.9981 and OA of 0.9993. Compared to nonlinear EKGCSC with the same parameter setting, which also reaches perfect scores, the proposed method matches the state-of-the-art nonlinear accuracy but with better computational efficiency, taking 27.7 seconds versus EKGCSC’s 55.7 seconds, a roughly 50\% reduction in time. The proposed approach vastly outperforms classic linear baselines such as K-Means (ARI = 0.7611) and Spectral Clustering (ARI = 0.7831), with relative improvements of approximately 31-32\% in ARI.

\begin{table}[ht]
\centering
\resizebox{0.5\textwidth}{!}{
\label{tab:dataset_summary}
\begin{tabular}{cccc}
\hline
Datasets & SalinasA & Indian Pines & Pavia University \\
\hline
Pixels   & $83 \times 86$   & $85 \times 70$    & $140 \times 150$ \\
Channels & 204              & 200               & 103              \\
Clusters & 6                & 4                 & 8                \\
Samples  & 5348             & 4391              & 6445             \\
Sensor   & AVIRIS           & AVIRIS            & ROSIS            \\
\hline
\end{tabular}
}
\caption{Summary of SalinasA, Indian Pines, and Pavia University datasets.}
\label{tab:summary}
\end{table}

\begin{table*}[ht]
\centering
\resizebox{1\textwidth}{!}{
\scriptsize
\begin{tabular}{|c|c|c|c|c|c|c|c|c|c|c|}
\hline
Method & ARI & NMI & OA & APR & APR C1 & APR C2 & APR C3 & APR C4 & time (s) \\
\hline
Proposed SCoS ($s_r=3$, $s_a=3$)           & \textbf{0.9219} & \textbf{0.9188} & \textbf{0.9743} & \textbf{0.9812} & 0.9656 & \textbf{1}       & \textbf{1}        & \textbf{0.9594} & 38.987  \\
ENSC                & 0.4548 & 0.4867 & 0.7498 & 0.7657 & 0.8469 & \textbf{0.9944}  & 0.9802 & 0.7277 & 11.7159 \\
BDR-Z               & 0.4192 & 0.4688 & 0.7181 & 0.7459 & 0.9275  & 0.9565  & 0.4325   & 0.6670 & 52.4709  \\
K-Means             & 0.3133 & 0.4304 & 0.6284 & 0.6994 & 0.8738  & 0.9264  & 0.3158   & 0.6818 & 0.962   \\
Spectral Clustering & 0.4616 & 0.4950 & 0.7566 & 0.7644 & 0.8365  & 0.9347  & 0.5578   & 0.7284 & 1.976   \\
S4C                 & 0.3809 & 0.4517 & 0.7045 & 0.7231 & 0.8039  & 0.9466  & 0.4208   & 0.7209 & 1.998   \\
Spectral CoS                & 0.3241 & 0.4162 & 0.6591 & 0.6882 & 0.4805  & 0.9847  & 0.6046   & 0.6830 & 2.9591   \\
EGCSC ($s_r=3$, $s_c=30$)           & 0.3222 & 0.4397 & 0.6566 & 0.5914 & \textbf{1}       & 0.7775  & 0        & 0.5880 & 12.971  \\
EKGCSC ($s_r=3$, $s_c=30$)          & 0.4264 & 0.5711 & 0.7167 & 0.6787 & 0.9712  & \textbf{0.9877}  & 0.1287   & 0.6273 & 19.056  \\
EGCSC ($s_r=9$, $s_c=30$)          & \textbf{0.6337} & \textbf{0.6373} & \textbf{0.8447} & \textbf{0.8868} & \textbf{0.9765}  & 0.7636  & \textbf{0.9981}   & \textbf{0.8089} & 12.465  \\
EKGCSC ($s_r=9$, $s_c=30$)         & \textbf{0.6829} & \textbf{0.6960} & \textbf{0.8700} & \textbf{0.9068} & \textbf{0.9725}  & 0.8391  & \textbf{0.9929}   & \textbf{0.8226} & 14.434  \\
\hline
\end{tabular}
}
\caption{Clustering performance metrics across methods on Indian Pines Dataset. Bold values indicate the top three performing methods for each metric. The parameters $s_{r}$ and $s_a$ refer to neighborhood sizes used in the representation and clustering purposes. Specifically, $s_r$ denotes the size of the $n \times n$ pixel neighborhood used to construct each pixel's representation, while $s_a$ refers to the patch size $n \times n$ of the neighborhood employed in building the adjacency matrix. $s_{\text{neigh}}$ is also used for the construction of adjacency matrices based on highly correlated pixels.} 
\label{tab:indian}
\end{table*}

\begin{table*}[ht]
\centering
\resizebox{1\textwidth}{!}{
\scriptsize
\setlength{\tabcolsep}{3pt}
\begin{tabular}{|c|c|c|c|c|c|c|c|c|c|c|c|c|c|c|}
\hline
Method & ARI & NMI & OA & APR & APR C1 & APR C2 & APR C3 & APR C4 & APR C5 & APR C6 & APR C7 & APR C8 & time (s) \\
\hline
 Proposed SCoS ($s_r=3$, $s_a=5$) & \textbf{0.9320} & \textbf{0.9029} & \textbf{0.9401} & \textbf{0.8342} & \textbf{0.9487} & \textbf{0.9783} & \textbf{1.0000} & \textbf{1.0000} & \textbf{0.9778} & 0.7685 & 0 & \textbf{1.0000} & 100.492 \\
ENSC & 0.4109 & 0.4914 & 0.6251 & 0.4514 & 0.1645 & 0.5161 & 0.0992 & 0.9953 & 0.8829 & 0.4666 & 0 & 0.4866 & 7.23408 \\
BDR-Z & 0.5445 & 0.6416 & 0.6939 & 0.5886 & 0.4047 & 0.5603 & 0 & 0.9924 & 0.9006 & 0.8917 & 0.1001 & 0.8585 & 34.8317 \\
K-Means & 0.4667 & 0.5664 & 0.6583 & 0.5512 & \textbf{0.3780} & 0.4173 & 0 & 0.9575 & 0.8029 & \textbf{0.8357} & \textbf{0.0972} & 0.9214 & 0.9676 \\
Spectral Clustering & 0.3052 & 0.5359 & 0.5105 & 0.4627 & 0.0431 & 0.3600 & 0.0778 & 0.9957 & 0.7383 & 0.5679 & 0 & 0.9186 & 2.8449 \\
S4C & 0.4831 & 0.5965 & 0.6416 & 0.5658 & 0.2664 & 0.5980 & 0 & 1.0000 & 0.8575 & \textbf{0.7860} & \textbf{0.1074} & 0.9144 &  9.98565 \\
Spectral CoS & 0.3179 & 0.4049 & 0.4922 & 0.4344 & 0.1838 & 0.5436 & 0.0030 & 0.9759 & 0.8566 & 0.7615 & 0.0202 & 0.1310 &  12.2659  \\
EGCSC ($s_r=9$, $s_c=30$) & \textbf{0.8129} & \textbf{0.8401} & \textbf{0.8442} & \textbf{0.6640} & 0 & \textbf{0.6888} & \textbf{1.0000} & \textbf{1.0000} & \textbf{0.9996} & 0.6280 & 0 & \textbf{0.9953} & 27.6136 \\
EKGCSC ($s_r=9$, $s_c=30$) & \textbf{0.9798} & \textbf{0.9488} & \textbf{0.9586} & \textbf{0.8453} & \textbf{0.8343} & \textbf{0.9808} & \textbf{1.0000} & \textbf{1.0000} & \textbf{1.0000} & \textbf{0.9471} & 0 & \textbf{1.0000} & 47.0074 \\
\hline
\end{tabular}}
\caption{Clustering performance metrics across methods on University of Pavia Dataset. Bold values indicate the top three performing methods for each metric. The parameters $s_r$ and $s_a$ refer to neighborhood sizes used in the representation and clustering purposes. Specifically, $s_r$ denotes the size of the $n \times n$ pixel neighborhood used to construct each pixel's representation, while $s_a$ refers to the patch size $n \times n$ of the neighborhood employed in building the adjacency matrix. $s_c$ is also used for the construction of adjacency matrices based on highly correlated pixels.}
\label{tab:pavia}
\end{table*}

\begin{table*}[h]
\centering
\resizebox{1\textwidth}{!}{
\scriptsize
\begin{tabular}{|c|c|c|c|c|c|c|c|c|c|c|c|}
\hline
{Method} & {ARI} & NMI & {OA} & {APR} & {APR C1} & {APR C2} & {APR C3} & {APR C4} & {APR C5} & {APR C6} & {time (s)} \\
\hline
Proposed SCoS  ($s_r=3$, $s_a=3$)  & \textbf{1.0000} & \textbf{1.0000} & \textbf{1.0000} & \textbf{1.0000} & \textbf{1.0000} & \textbf{1.0000} & \textbf{1.0000} & \textbf{1.0000} & \textbf{1.0000} & \textbf{1.0000} & 27.6474 \\
ENSC                & 0.7629 & 0.8319 & 0.8364 & 0.8887 & 1 & 0.9928 & 0.4251 & 0.9794 & 0.9924 & 0.9427 & 35.4402 \\
BDR-Z               & 0.7699 & 0.8363 & 0.8743 & 0.9079 & 1 & 0.9937 & 0.9909 & 0.9299 & 0.5407 & 0.9925 & 136.524 \\
K-Means             &  0.6669 & 0.7403  & 0.7061 & 0.6697 & 0 & 1 & 0.4251 & 0.9501 & 0.9978 & 0.6450 & 1.08062 \\
Spectral Clustering & 0.7831 & 0.8668 & 0.8435 & 0.8967 & \textbf{1.0000} & 0.9928 & 0.4273 & 0.9726 & 0.9985 & 0.9888 & 2.8153 \\
S4C                 & 0.7286 & 0.7964 & 0.8108 & 0.8704 & 0.9773 & \textbf{1.0000} & 0.4203 & 0.9356 & 1.0000 & 0.8891 & 3.3157 \\
Spectral CoS                & 0.3808 & 0.4596 & 0.5611 & 0.5258 & 0.4850 & 0.9172 & 0.3241 & 0.5978 & 0.7436 & 0.5978 & 4.6709 \\
EGCSC ($s_r=3$, $s_{\text{neigh}}=30$) & 0.7858 & 0.8623 & 0.8702 & 0.9000 & \textbf{1.0000} & \textbf{1.0000} & 0.8745 & 0.9967 & 0.5286 & \textbf{1.0000} & 16.2556  \\
EKGCSC ($s_r=3$, $s_c=30$) & 0.7849 & 0.8600 & 0.8725 & 0.9002 & \textbf{1.0000} & \textbf{1.0000} & 0.8699 & 0.9993 & 0.5345 & 0.9975 & 20.9916 \\
EGCSC ($s_r=9$, $s_c=30$) & \textbf{0.9981} & \textbf{0.9971} & \textbf{0.9993} & \textbf{0.9992} & \textbf{1.0000} & \textbf{1.0000} & \textbf{1.0000} & \textbf{1.0000} & \textbf{1.0000} & 0.9950 & 20.0051 \\
EKGCSC ($s_r=9$, $s_c=30$) & \textbf{1.0000} & \textbf{1.0000} & \textbf{1.0000} & \textbf{1.0000} & \textbf{1.0000} & \textbf{1.0000} & \textbf{1.0000} & \textbf{1.0000} & \textbf{1.0000} & \textbf{1.0000} & 32.2392 \\
\hline
\end{tabular}
}
\caption{Clustering performance metrics across methods on Salinas-A Dataset. Bold values indicate the top three performing methods for each metric. The parameters $s_r$ and $s_a$ refer to neighborhood sizes used in the representation and clustering purposes. Specifically, $s_r$ denotes the size of the $n \times n$ pixel neighborhood used to construct each pixel's representation, while $s_a$ refers to the patch size $n \times n$ of the neighborhood employed in building the adjacency matrix. $s_c$ is also used for the construction of adjacency matrices based on highly correlated pixels. }\label{tab:salinas}
\end{table*}

\section{Conclusions}
In this paper, we proposed Subspace Clustering of Subspaces, where the goal is to cluster a collection of tall matrices based on their column spaces. We established connections—and highlighted key differences—between this problem and the well-studied areas of Subspace Clustering and Generalized Canonical Correlation Analysis (GCCA). Our approach is grounded in the classic Block Term tensor Decomposition (BTD) formulation, along with an effective optimization algorithm for learning both the cluster memberships and the partially ``common" subspaces shared across clusters. Leveraging the connection to tensor decomposition, we also provided the first identifiability guarantees for this setting. To ensure scalability, we developed efficient algorithms suitable for large-scale datasets. We demonstrated the applicability of the proposed framework in hyperspectral imaging, specifically for pixel clustering. Through extensive experiments on both real-world hyperspectral and synthetic datasets, we compared our method against several state-of-the-art subspace clustering techniques. Finally, we showed that our approach achieves superior performance and robustness, even in the presence of significant noise and interference.

\bibliographystyle{IEEEtran}
\bibliography{IEEEabrv,refrences}

\section*{Supplementary Material}






\title{Supplementary Material of Subspace Clustering of Subspaces: Unifying Canonical Correlation Analysis and Subspace Clustering}

\author{Paris~A.~Karakasis,~\IEEEmembership{Graduate Student Member,~IEEE} and Nicholas~D.~Sidiropoulos,~\IEEEmembership{Fellow,~IEEE}%
\thanks{Paris~A.~Karakasis and Nicholas~D.~Sidiropoulos are with the Department of Electrical and Computer Engineering, University of Virginia, Charlottesville, VA 22904 USA (e-mail: $\left\{\text{karakasis,~nikos}\right\}$@virginia.edu).}
}





\section{Determining the Dimensions of the Partially Common Subspaces-Algorithmic Description}

In this section, we provide a detailed algorithmic description of the procedure proposed to estimate the dimensions and bases of the partially common subspaces introduced in the main paper. As discussed therein, given a fixed assignment of data matrices to clusters, the estimation of each subspace can be performed independently. This supplement formalizes the computational strategy for doing so efficiently, depending on the ambient dimension $N$ and the dimensions of the participating matrices.

The key quantity that must be computed for each cluster is the set of eigenvalues of the matrix $\mathbf{T}_r$. As explained in the main manuscript, this can be accomplished either by explicitly forming $\mathbf{T}_r$ and computing its eigendecomposition, or by exploiting the potential low-rank structure of $\mathbf{T}_r$. The latter is achieved by stacking the basis matrices $\mathbf{U}_k$, for all $k \in \mathbb{I}_r$, side by side, and then computing the singular value decomposition (SVD) of the resulting matrix. Once the eigenvalues of $\mathbf{T}_r$ are available, the dimension $L_r^*$ of the $r$-th partially common subspace can be estimated using established model selection techniques. These include the Akaike Information Criterion (AIC), the Minimum Description Length (MDL), or other approaches such as the ones proposed in \cite{liavas1999blind} and \cite{karakasis2022multisubject}.

Algorithm~1 summarizes the proposed procedure for estimating the values $L_r^*$ for all clusters, given a fixed assignment of samples to clusters.

\begin{algorithm}
\caption{Dimension and Basis Estimation via Projectors or Low Dimensional Bases Stacking}
\begin{algorithmic}[1]
\Input Matrices $\{\mathbf{U}_k \in \mathbb{R}^{N \times M_k}\}_{k=1}^K$ and assignment vector $\mathbf{z} \in [R]^K$
\Output Estimated dimensions $L_r^*$ and bases $\mathbf{G}_r \in \mathbb{R}^{N \times L_r^*}$ for all $r \in [R]$

\For{each cluster $r = 1$ to $R$}
    \State $\mathbb{I}_r \gets \{k \in [K] ~|~ \mathbf{z}_k = r\}$
    \State $S_r \gets \sum_{k \in \mathbb{I}_r} M_k$
    
    \If{$N > S_r^2$} \Comment{Low-rank case — use stacking}
        \State Form $\bar{\mathbf{U}}_{\mathbb{I}_r} = [\ldots~\mathbf{U}_k~\ldots]_{k \in \mathbb{I}_r} \in \mathbb{R}^{N \times S_r}$
        \State Compute SVD of $\bar{\mathbf{U}}_{\mathbb{I}_r} = \mathbf{U}_r \boldsymbol{\Sigma}_r \mathbf{V}^T_r$
        \State $\boldsymbol{\Lambda}_r=\boldsymbol{\Sigma}_r^2$
    \Else \Comment{Use sum of projectors}
        \State $\mathbf{T}_r \gets \sum_{k \in \mathbb{I}_r} \mathbf{U}_k \mathbf{U}_k^T$
        \State Compute eigendecomposition: $\mathbf{T}_r = \mathbf{U}_r \boldsymbol{\Lambda}_r \mathbf{U}_r^T$
    \EndIf

    \State Estimate $L_r^*$ from $\boldsymbol{\Lambda}_r$ using model selection (e.g., AIC, MDL, \cite{liavas1999blind}, \cite{karakasis2022multisubject})
    \State Set $\mathbf{G}_r \gets$ first $L_r^*$ columns of $\mathbf{U}_r$
\EndFor
\end{algorithmic}
\end{algorithm}

\section{Segmentation Maps from Hyperspectral Image Pixel Clustering}

This section presents the segmentation maps produced by all evaluated methods on the three hyperspectral datasets. In Figs. \ref{fig:seg_maps_indian_pines}, \ref{fig:seg_maps_univ_pavia}, and \ref{fig:seg_maps_salinas}, the reader may find the segmentation maps for the Indian Pines, University of Pavia, and Salinas-A datasets, respectively. In all cases, our method based on Subspace Clustering of Subspaces produces segmentations of high clarity and coherence—closely matching the ground truth in the Indian Pines and University of Pavia datasets, and achieving perfect clustering of the pixels in the Salinas-A dataset.

\begin{figure*}[htbp]
  \centering
  \includegraphics[trim=129pt 37pt 80pt 10pt, clip, width=0.25\linewidth]{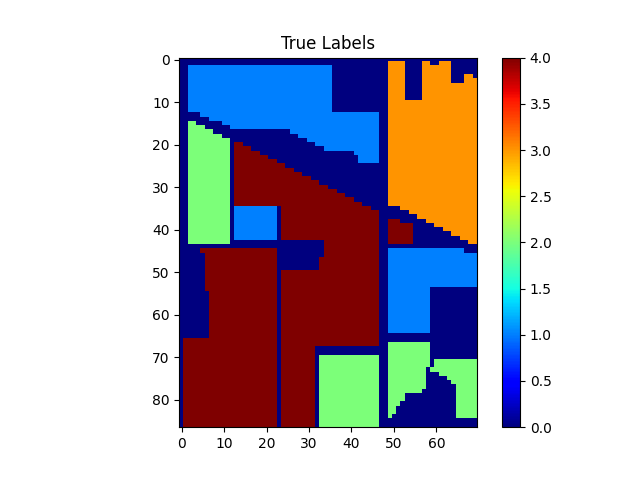}\hspace{-6mm}
  \includegraphics[trim=129pt 37pt 80pt 10pt, clip, width=0.25\linewidth]{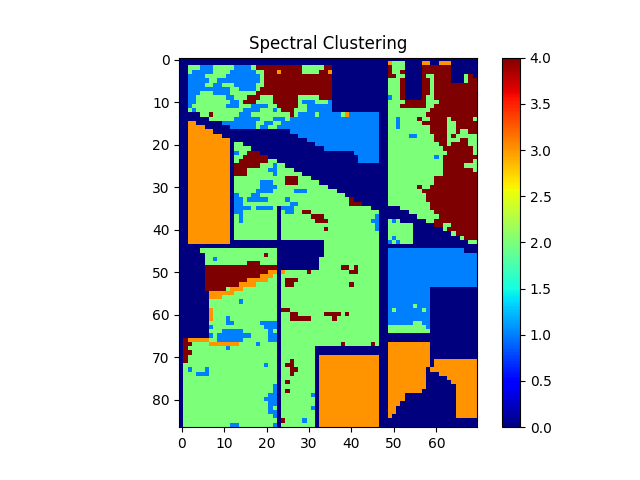}\hspace{-6mm}
  \includegraphics[trim=129pt 37pt 80pt 10pt, clip, width=0.25\linewidth]{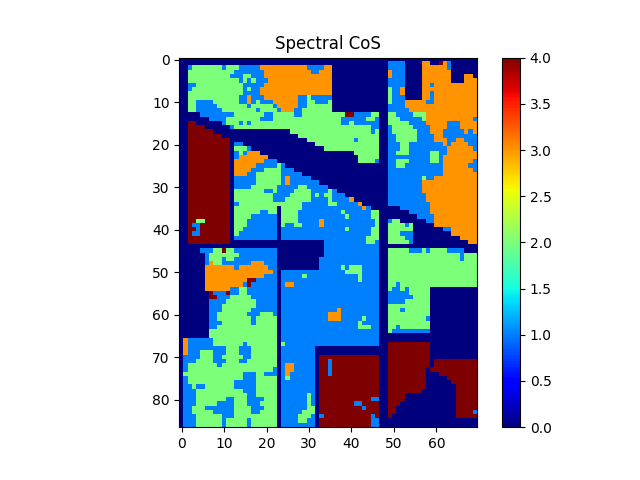}\hspace{-6mm}
  \raisebox{-0.4mm}{\includegraphics[trim=129pt 34pt 20pt 10pt, clip, width=0.309\linewidth]{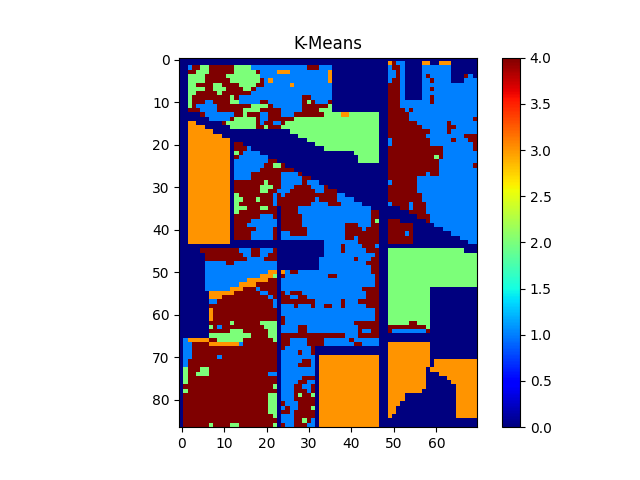}}

  \vspace{1mm}

  \includegraphics[trim=129pt 37pt 80pt 10pt, clip, width=0.25\linewidth]{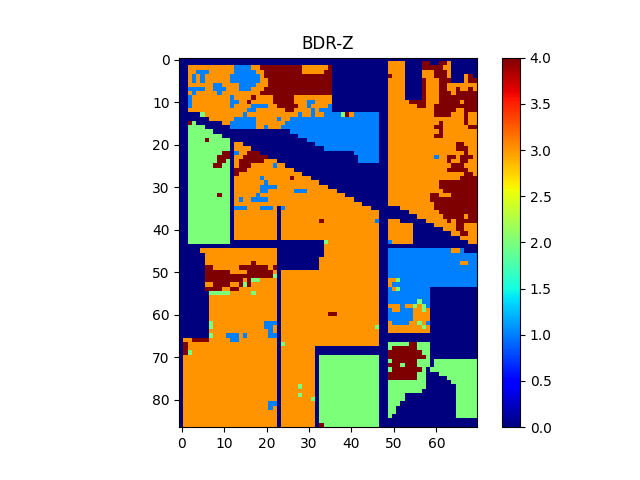}\hspace{-6mm}
  \includegraphics[trim=129pt 37pt 80pt 10pt, clip, width=0.25\linewidth]{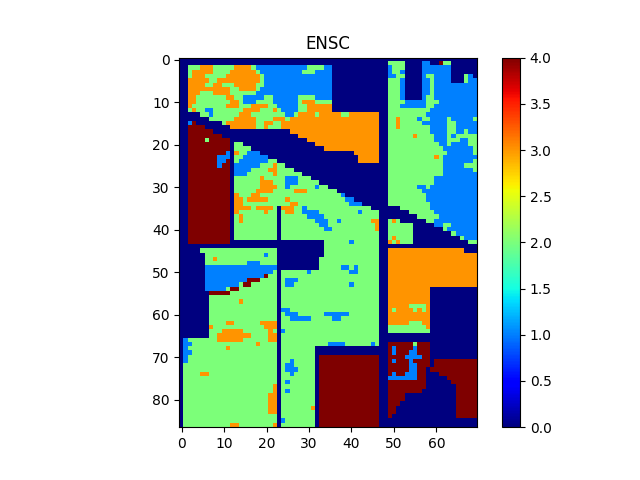}\hspace{-6mm}
  \includegraphics[trim=129pt 37pt 80pt 10pt, clip, width=0.25\linewidth]{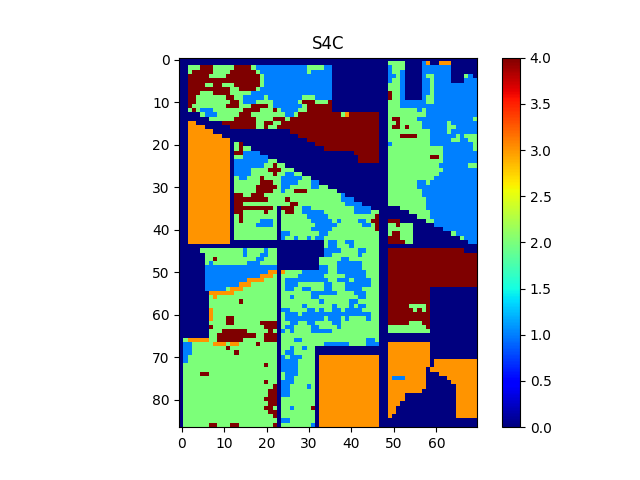}\hspace{-6mm}
   \raisebox{-0.5mm}{\includegraphics[trim=129pt 33pt 20pt 10pt, clip, width=0.308\linewidth]{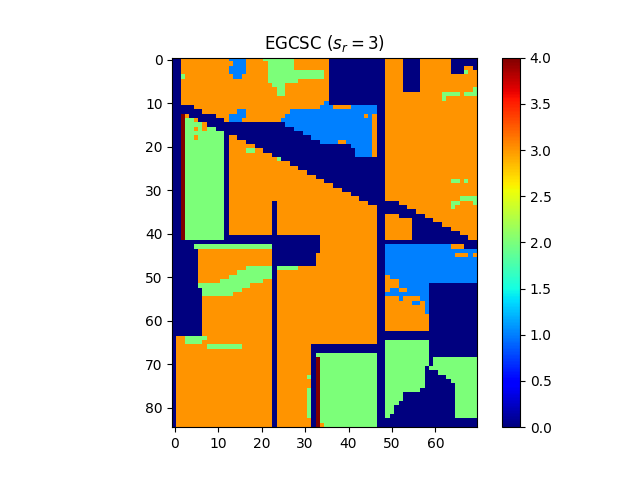}}

  \vspace{1mm}

  \includegraphics[trim=129pt 37pt 80pt 10pt, clip, width=0.25\linewidth]{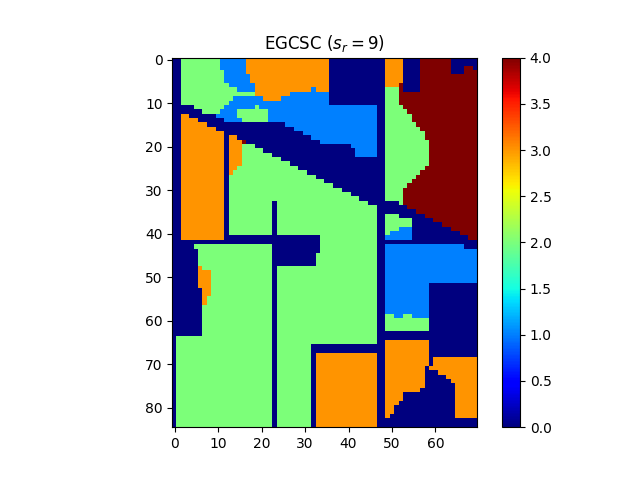}\hspace{-6mm}
  \includegraphics[trim=129pt 37pt 80pt 10pt, clip, width=0.25\linewidth]{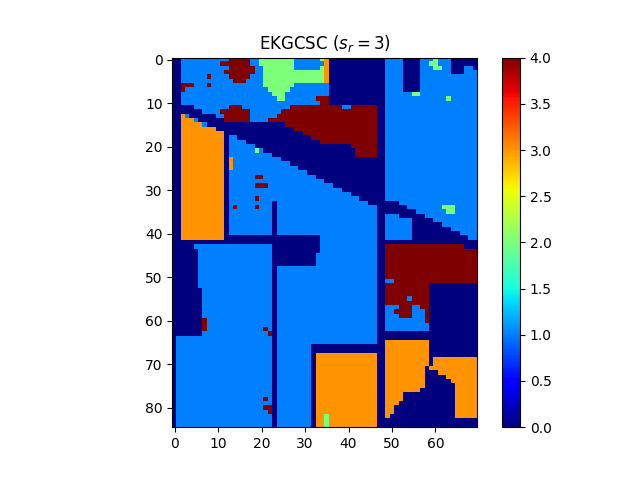}\hspace{-6mm}
  \includegraphics[trim=129pt 37pt 80pt 10pt, clip, width=0.25\linewidth]{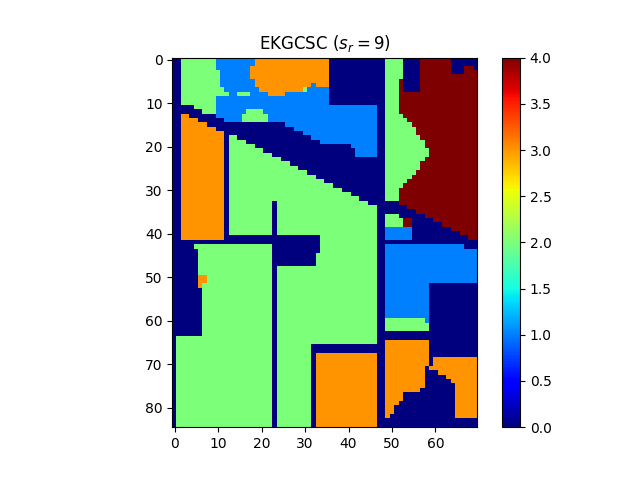}\hspace{-6mm}
   \raisebox{-0.5mm}{\includegraphics[trim=129pt 33pt 20pt 10pt, clip, width=0.308\linewidth]{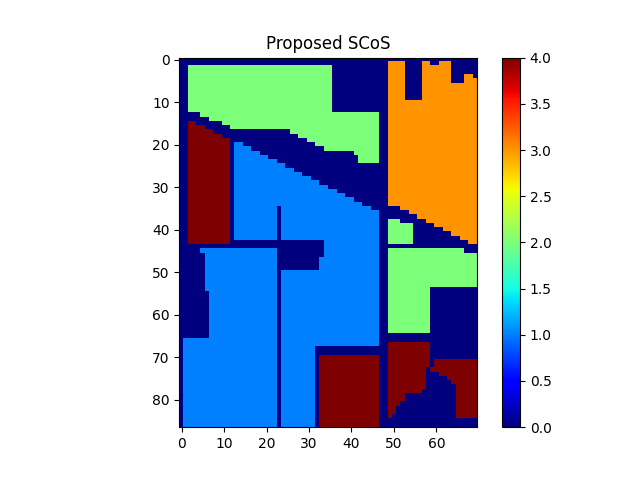}}
  
  \caption{Segmentation maps of the Indian Pines dataset produced by all considered methods.}
  \label{fig:seg_maps_indian_pines}
\end{figure*}

\begin{figure*}[htbp]
  \centering
  \includegraphics[trim=210pt 37pt 90pt 10pt, clip, width=0.21\linewidth]{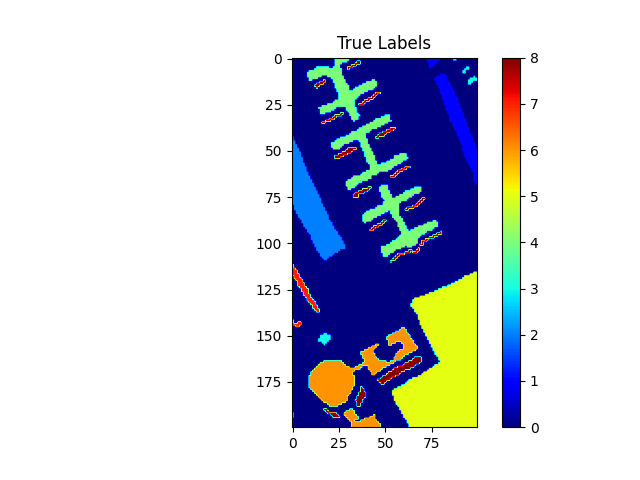}\hspace{-7mm}
  \includegraphics[trim=210pt 37pt 90pt 10pt, clip, width=0.21\linewidth]{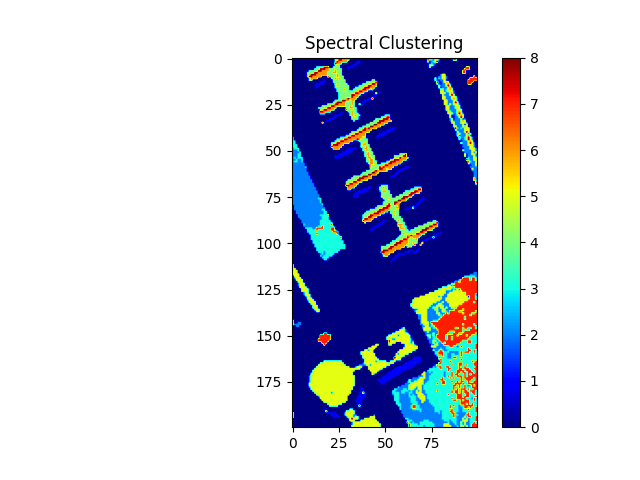}\hspace{-7mm}
  \includegraphics[trim=210pt 37pt 90pt 10pt, clip, width=0.21\linewidth]{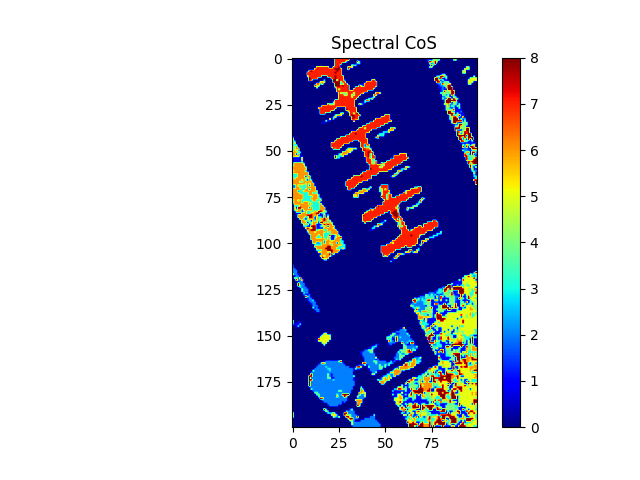}\hspace{-7mm}
  \includegraphics[trim=210pt 37pt 90pt 10pt, clip, width=0.21\linewidth]{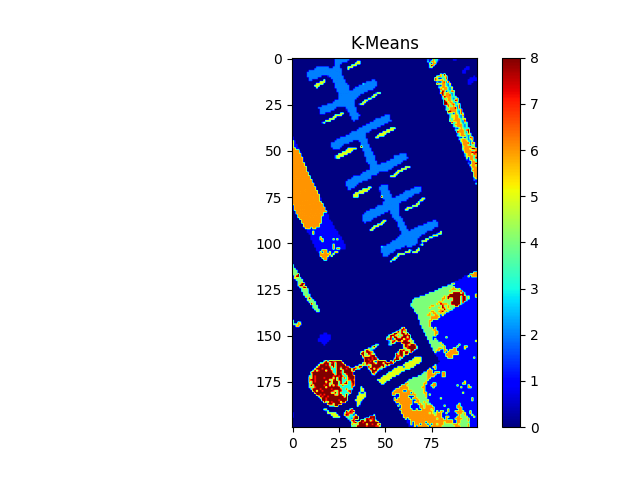}\hspace{-7mm}
 \raisebox{-0.75mm}{\includegraphics[trim=210pt 33pt 40pt 10pt, clip, width=0.274\linewidth]{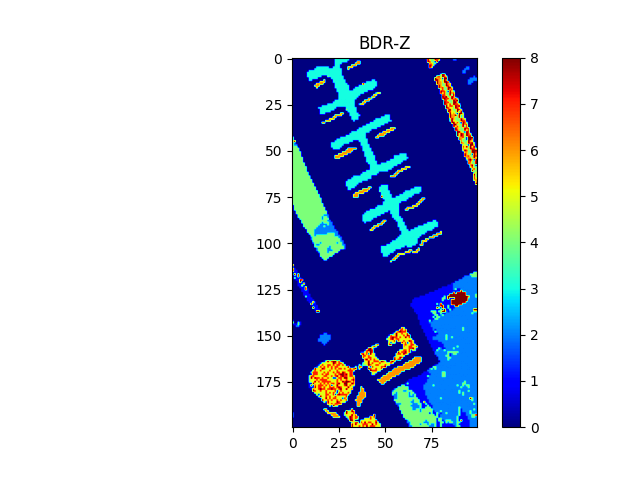}}
 
  \vspace{1mm}

  \includegraphics[trim=210pt 37pt 90pt 10pt, clip, width=0.21\linewidth]{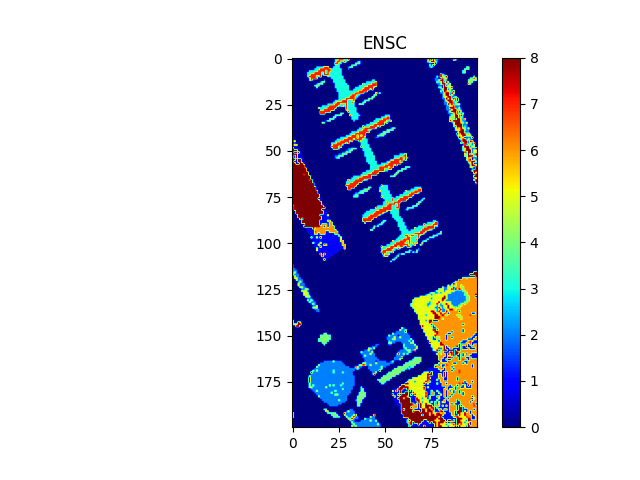}\hspace{-7mm}
  \includegraphics[trim=210pt 37pt 90pt 10pt, clip, width=0.21\linewidth]{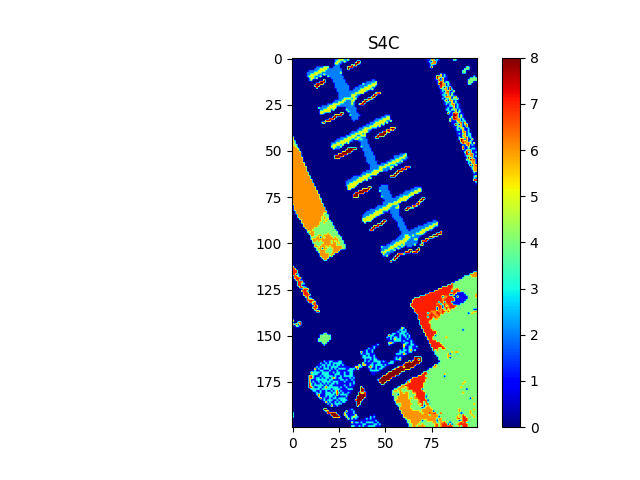}\hspace{-7mm}
  \includegraphics[trim=210pt 37pt 90pt 10pt, clip, width=0.21\linewidth]{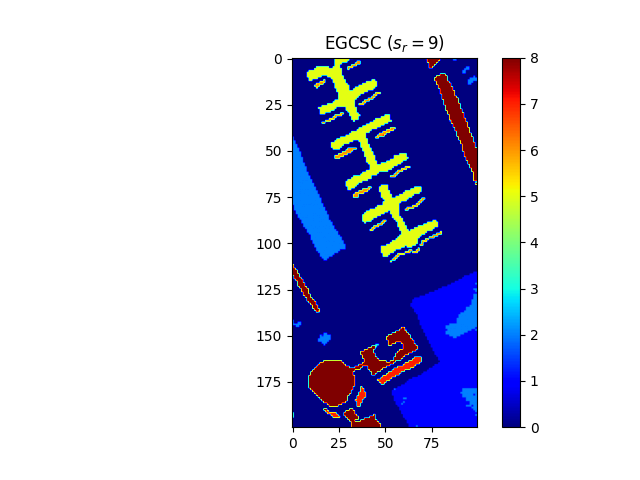}\hspace{-7mm}
  \includegraphics[trim=210pt 37pt 90pt 10pt, clip, width=0.21\linewidth]{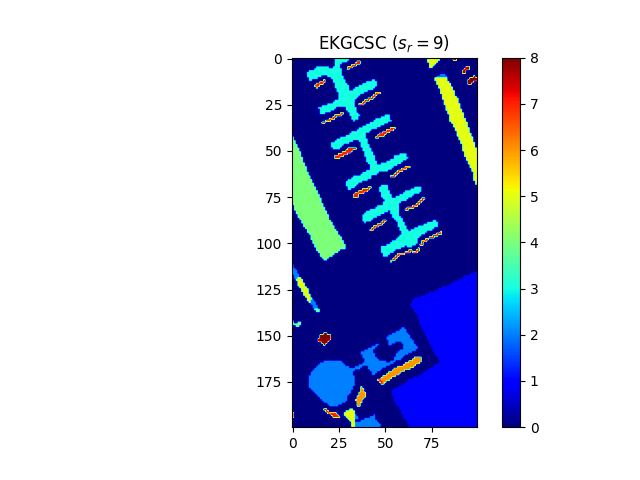}\hspace{-7mm}
   \raisebox{-0.75mm}{\includegraphics[trim=210pt 33pt 40pt 10pt, clip, width=0.274\linewidth]{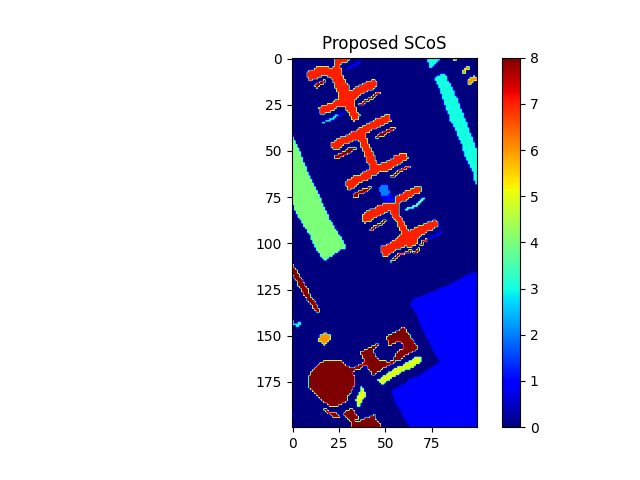}}
  
  \caption{Segmentation maps of the University of Pavia dataset produced by all considered methods.}
  \label{fig:seg_maps_univ_pavia}
\end{figure*}

\begin{figure*}[htbp]
  \centering
  \includegraphics[trim=68pt 38pt 80pt 10pt, clip, width=0.25\linewidth]{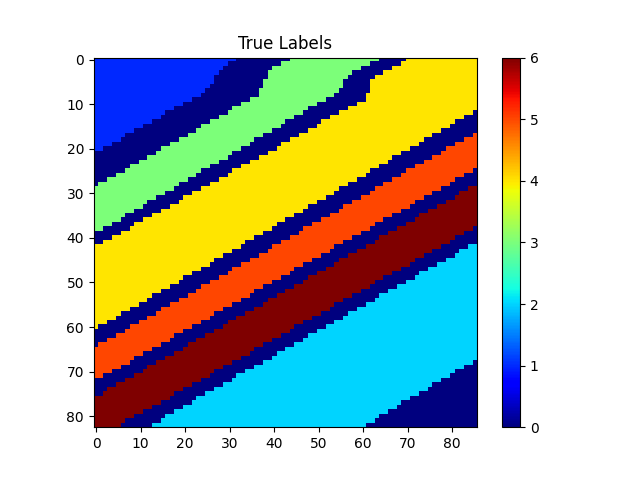}\hspace{-6mm}
  \includegraphics[trim=68pt 38pt 80pt 10pt, clip, width=0.25\linewidth]{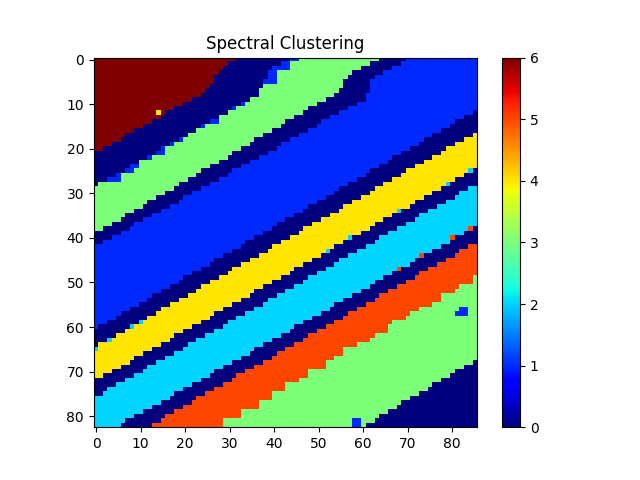}\hspace{-6mm}
  \includegraphics[trim=68pt 38pt 80pt 10pt, clip, width=0.25\linewidth]{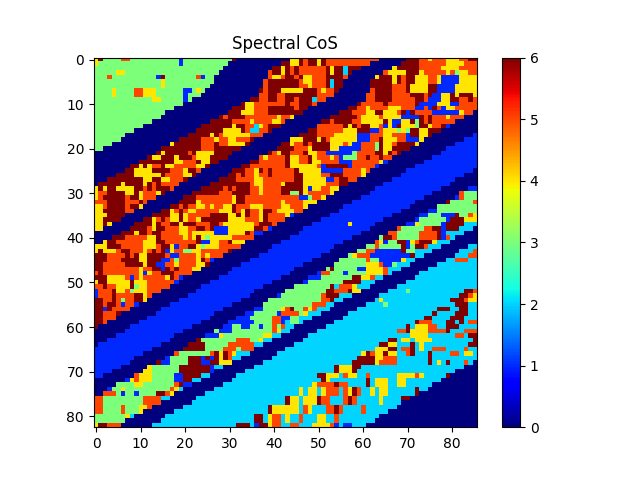}\hspace{-6mm}
  \raisebox{-0.475mm}{\includegraphics[trim=68pt 33pt 20pt 10pt, clip, width=0.297\linewidth]{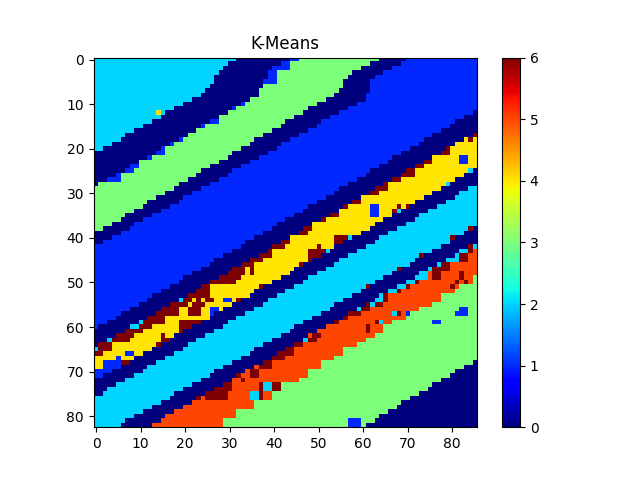}}

  \vspace{1mm}

  \includegraphics[trim=68pt 37pt 80pt 10pt, clip, width=0.25\linewidth]{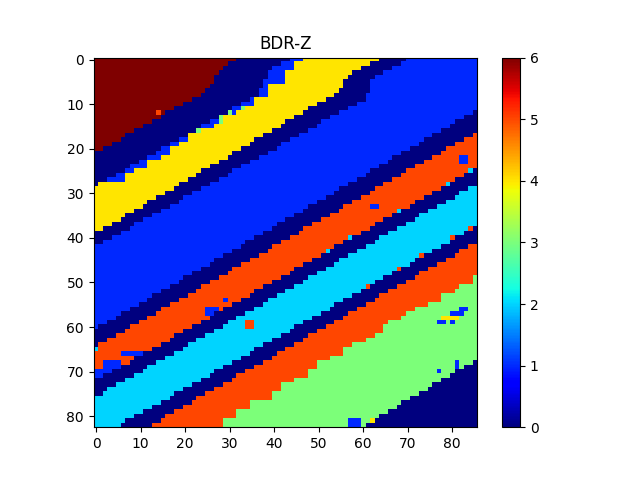}\hspace{-6mm}
  \includegraphics[trim=68pt 37pt 80pt 10pt, clip, width=0.25\linewidth]{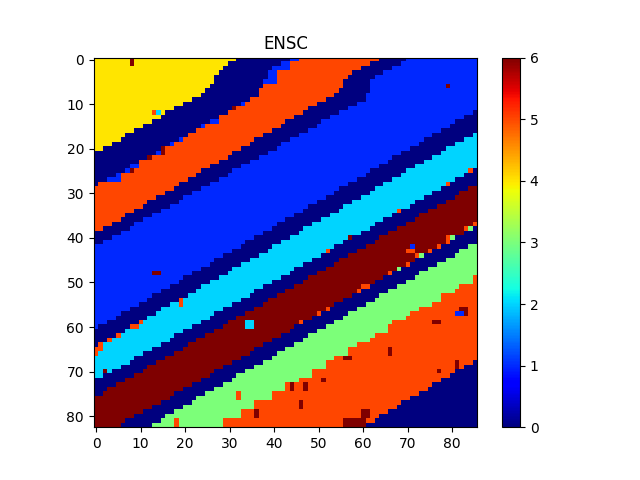}\hspace{-6mm}
  \includegraphics[trim=68pt 37pt 80pt 10pt, clip, width=0.25\linewidth]{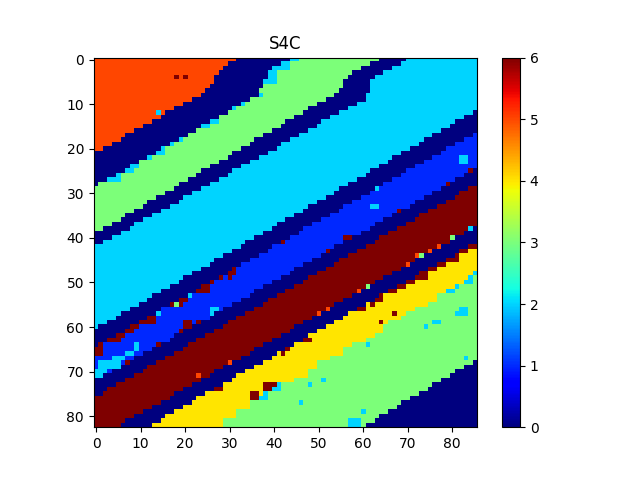}\hspace{-6mm}
   \raisebox{-0.475mm}{\includegraphics[trim=68pt 33pt 20pt 10pt, clip, width=0.297\linewidth]{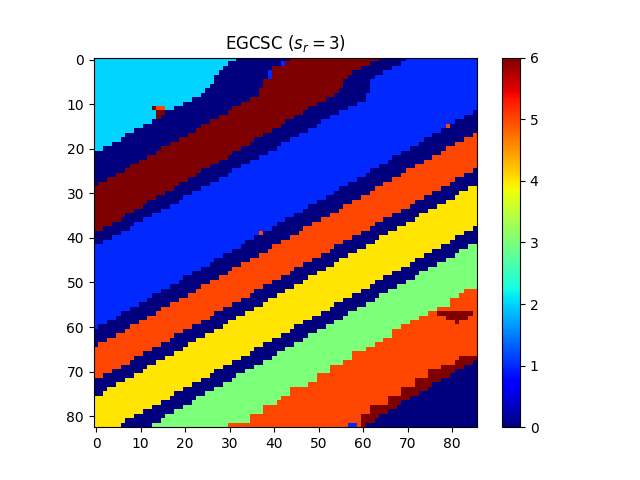}}

  \vspace{1mm}

  \includegraphics[trim=68pt 37pt 80pt 10pt, clip, width=0.25\linewidth]{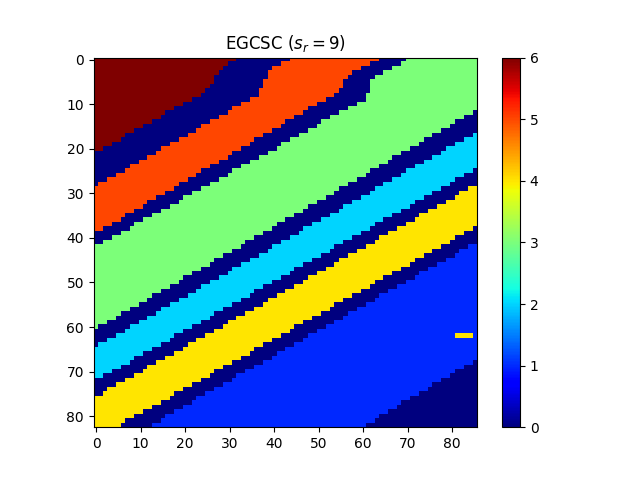}\hspace{-6mm}
  \includegraphics[trim=68pt 37pt 80pt 10pt, clip, width=0.25\linewidth]{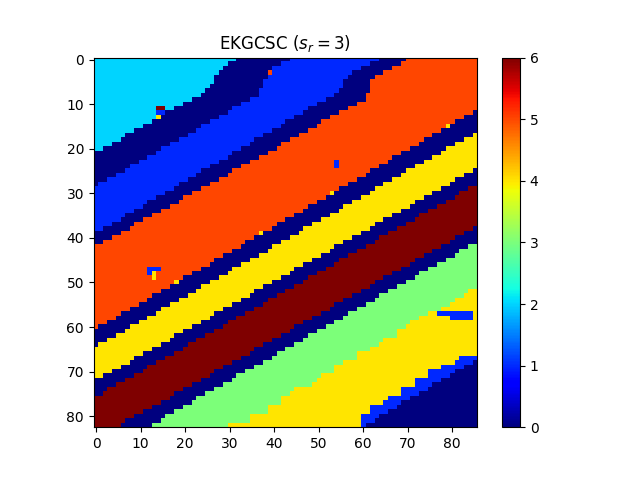}\hspace{-6mm}
  \includegraphics[trim=68pt 37pt 80pt 10pt, clip, width=0.25\linewidth]{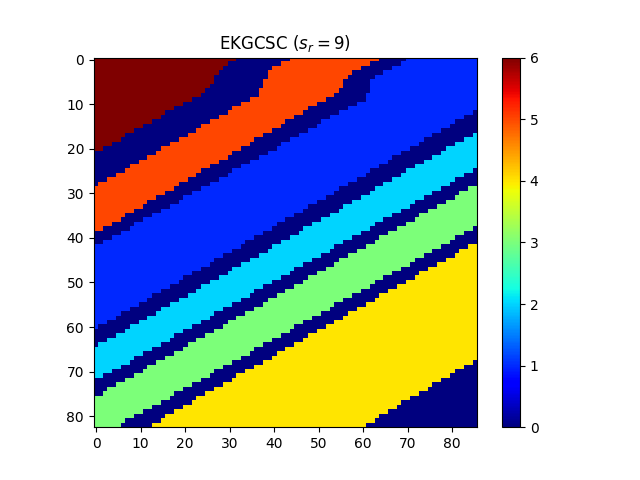}\hspace{-6mm}
   \raisebox{-0.5mm}{\includegraphics[trim=68pt 33pt 20pt 10pt, clip, width=0.297\linewidth]{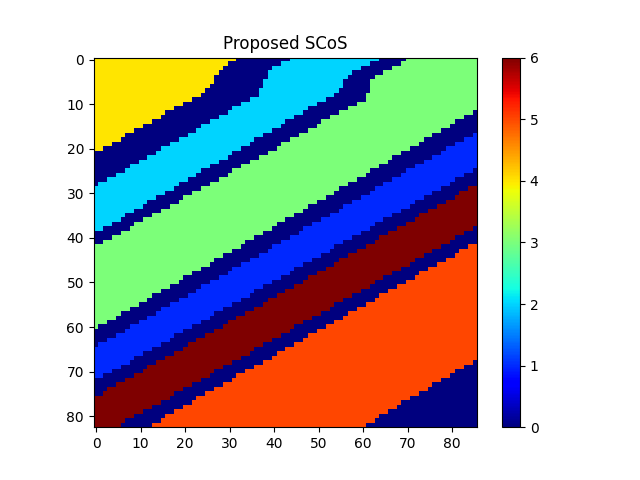}}
  
  \caption{Segmentation maps of the Salinas-A dataset produced by all considered methods.}
  \label{fig:seg_maps_salinas}
\end{figure*}

\section{Experiments with synthetic data}

In this section, we compare our method, and specifically the algorithm we proposed in \cite{karakasis2023clustering}, to three popular subspace clustering methods: Sparse Subspace Clustering OMP (SSC-OMP) \cite{you2016scalable}, Elastic Net Subspace Clustering (ENSC) \cite{you2016oracle}, and Block Diagonal Representation based Subspace clustering (BDRSC) \cite{lu2018subspace}. We also use two classical clustering methods, namely, K-means and Spectral Clustering, as additional baselines.

For our experiments, we use synthetic data and we consider four different scenarios. For all scenarios, we consider five different partially common subspaces of dimension equal to $L_r=L=20$ in $\mathbb{R}^{1000}$. We set the number of views to $K=100$, the number of columns of each view to $M=50$, and we assign each view, $\mathbf{X}_k$, to one of the clusters (we denote it as $r_k$) equiprobably. Then, the data matrix of each view is generated according to the generative model in relation (5) of the main manuscript with additive normal i.i.d. noise, i.e. 
\begin{equation}
    \mathbf{X}_k = \mathbf{G}_{r_k}\mathbf{Q}_k^{(1)}+\mathbf{H}_k\mathbf{Q}_k^{(2)}+\mathbf{E}_k,
\end{equation}
where $\mathbf{Q}_k^{(1)}\in\mathbb{R}^{L\times M}$, $\mathbf{Q}_k^{(2)}\in\mathbb{R}^{M-L\times M}$, $\mathbf{G}_{r_k}\in\mathbb{R}^{N\times L}$ with $\mathbf{G}_{r_k}^T\mathbf{G}_{r_k}=\mathbf{I}_L$, $\mathbf{H}_{k}\in\mathbb{R}^{N\times M-L}$ with $\mathbf{H}_{k}^T\mathbf{H}_{k}=\mathbf{I}_{M-L}$, while we also ensure that $\text{col}\left(\mathbf{G}_{r_k}\right)\cap\text{col}\left(\mathbf{H}_k\right)=\emptyset$. Matrices $\mathbf{Q}_k^{(1)}\in\mathbb{R}^{L\times M}$, $\mathbf{Q}_k^{(2)}\in\mathbb{R}^{M-L\times M}$, and $\mathbf{E}_k\in\mathbb{R}^{N\times M}$ are drawn to have normal i.i.d. elements. After letting 
\begin{equation}
    \text{INR} = \text{INR}_k:= \frac{\left\|\mathbf{H}_k\mathbf{Q}_k^{(2)}\right\|_F^2}{\left\|\mathbf{E}_k\right\|_F^2}
\end{equation}
and 
\begin{equation}
    \text{SINR} =  \text{SINR}_k:=\frac{\left\|\mathbf{G}_{r_k}\mathbf{Q}_k^{(1)}\right\|_F^2}{\left\|\mathbf{H}_k\mathbf{Q}_k^{(2)}+\mathbf{E}_k\right\|_F^2},
\end{equation}
we consider the values $\{0.5, \, 1, \,5, \,10\}$, one in each scenario, for $\text{INR}$, while in all scenarios, we let $\text{SINR}$ vary from $-15$\,dB to $5$\,dB.

Since the number of columns, $M$, is the same for each view, we apply the considered subspace clustering methods to the vectorized versions of the views $\mathbf{X}_k$. For the SSC-OMP, ENSC, and BDRSC methods, we used the implementations provided by the original authors.\footnote{See: \url{https://github.com/ChongYou/subspace-clustering}, \url{https://github.com/canyilu/Block-Diagonal-Representation-for-Subspace-Clustering}} The BDRSC method computes two matrices, $\mathbf{B}$ and $\mathbf{Z}$, which can be used interchangeably to construct the affinity matrix~\cite{lu2018subspace}. In our experiments, we use the matrix $\mathbf{Z}$, and we refer to this variant as BDR-Z. For K-Means and Spectral Clustering, we used the implementations provided by the \texttt{scikit-learn} Python package. Finally, the parameters of all models were manually tuned to maximize performance.

We compare all the methods in terms of clustering performance using the following three criteria, Clustering Accuracy (ACC), Normalized Mutual Information (NMI), and Adjusted Rand Index (ARI). More specifically, for each scenario and each value of SINR, we conduct $100$ Monte Carlo runs and report the averages together with the corresponding standard deviations for all the methods. In addition, we report the average walltime and the corresponding standard deviations. The results from each scenario appear in Figs \ref{fig:acc}, \ref{fig:ari}, \ref{fig:nmi}, and \ref{fig:times}.

We observe that in all scenarios, the proposed method consistently outperforms all baselines across all considered classification metrics and scenarios. In the high SINR regime, it achieves very high performance with negligible deviation, similar to ENSC and SSC-OMP. However, as SINR decreases, the performance of the baselines deteriorates significantly earlier than that of the proposed method, both in terms of average accuracy and variability. This trend holds consistently across all considered INR values. Specifically, the proposed method maintains peak performance with zero deviation even at SINRs as low as $-7.5\,\mathrm{dB}$, $-9\,\mathrm{dB}$, $-11\,\mathrm{dB}$, and $-15\,\mathrm{dB}$ for INR values of $0.5$, $1$, $5$, and $10$, respectively. In contrast, the baselines begin to degrade once SINR drops below $-3.5\,\mathrm{dB}$, regardless of the INR or metric. For SINRs below $-7.5\,\mathrm{dB}$, ACC falls below $0.65$, and both ARI and NMI drop below $0.4$, for all other methods. Finally, regarding computational efficiency, the proposed method demonstrates comparable or lower runtime in the SINR range where all methods perform well, and requires more time only in extremely low SINRs---where all baselines already fail to perform adequately.

\begin{figure}[t]
    \centering
    \includegraphics[width=0.42\textwidth, trim=20 10 30 10, clip]{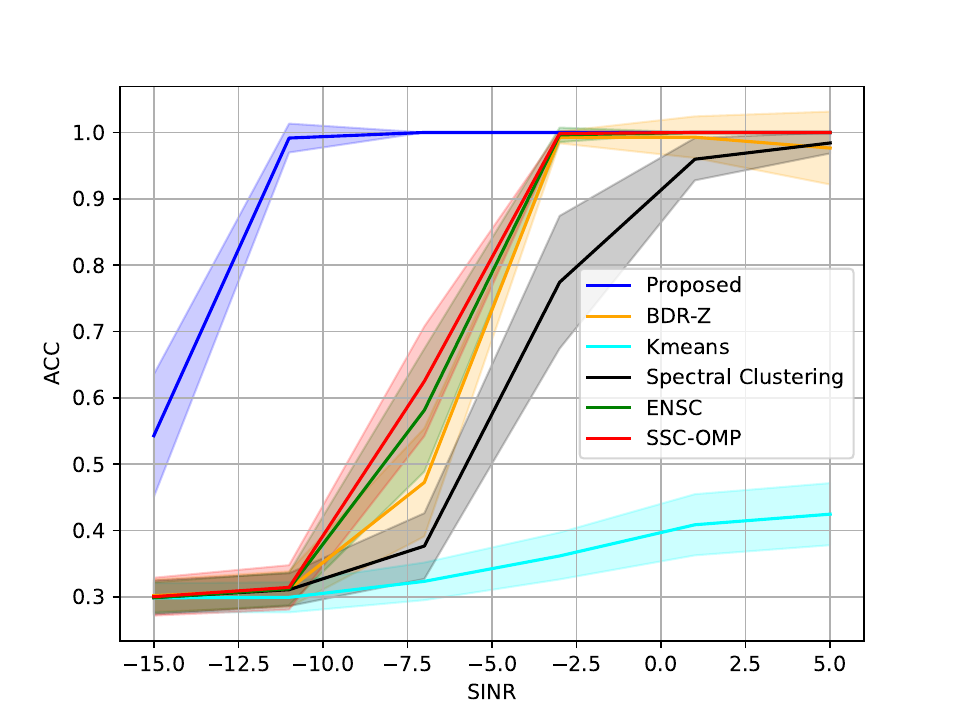}
    \includegraphics[width=0.42\textwidth, trim=20 10 30 10, clip]{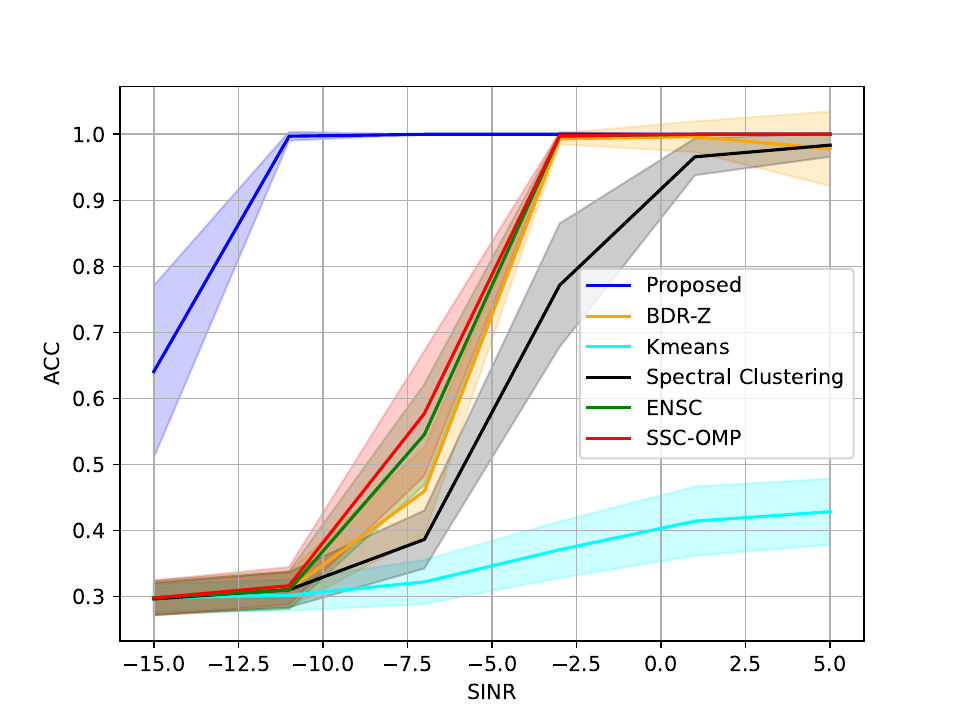}
    \includegraphics[width=0.42\textwidth, trim=20 10 30 10, clip]{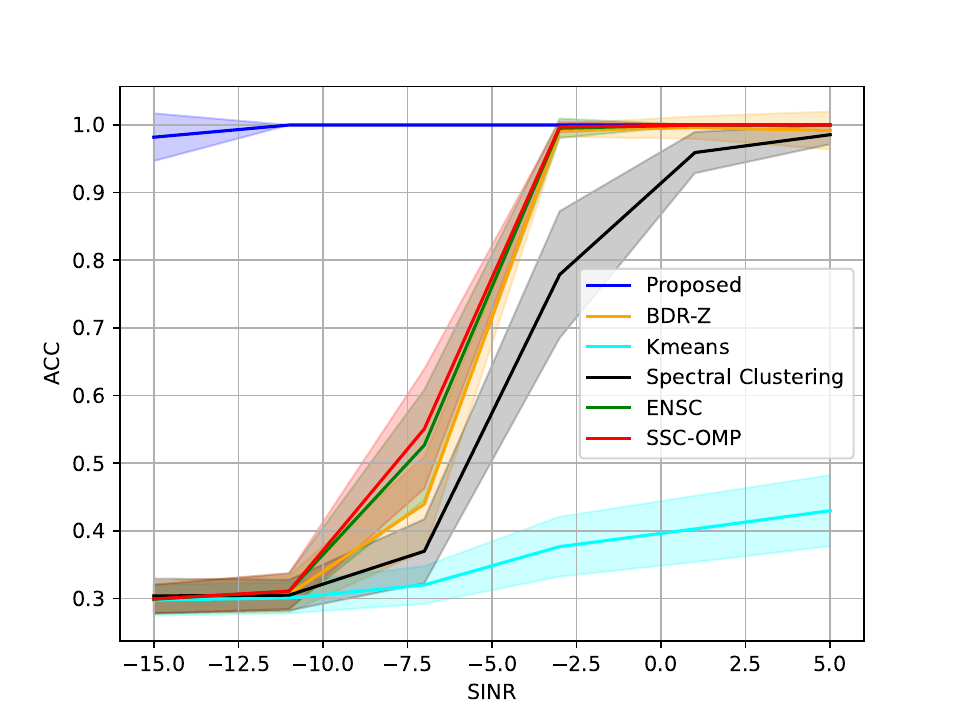}
    \includegraphics[width=0.42\textwidth, trim=20 10 30 10, clip]{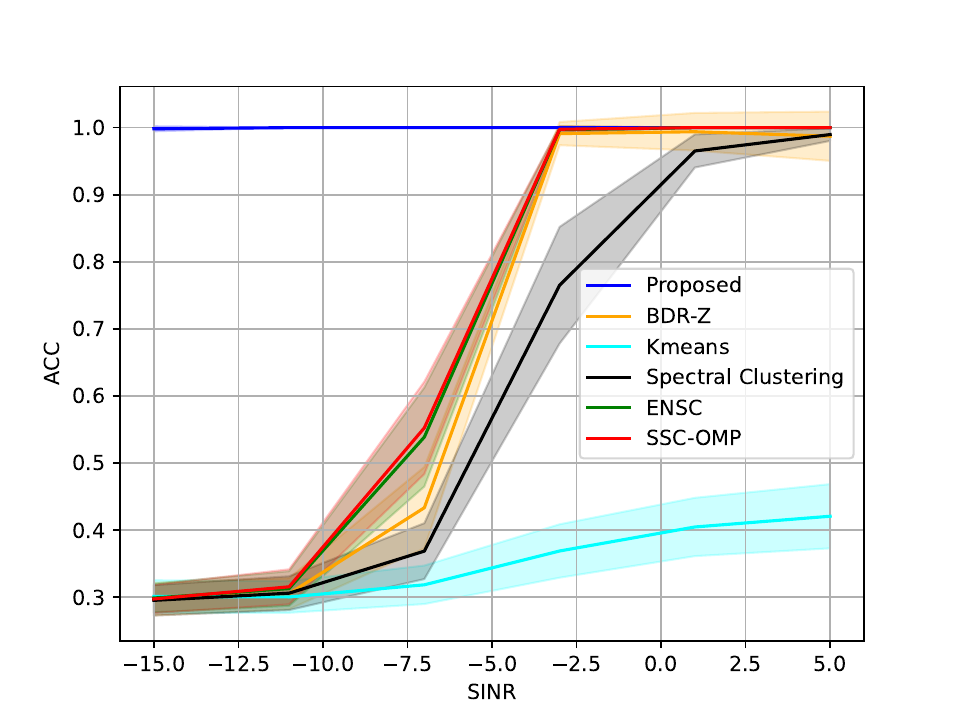}
    \caption{Average ACC metric over SINR for INR = $\left\{0.5,~1,~5,~10\right\}$ (starting from above), based on $100$ Monte Carlo runs.}
    \label{fig:acc}
\end{figure}

\begin{figure}[t]
    \centering
    \includegraphics[width=0.42\textwidth, trim=20 10 30 10, clip]{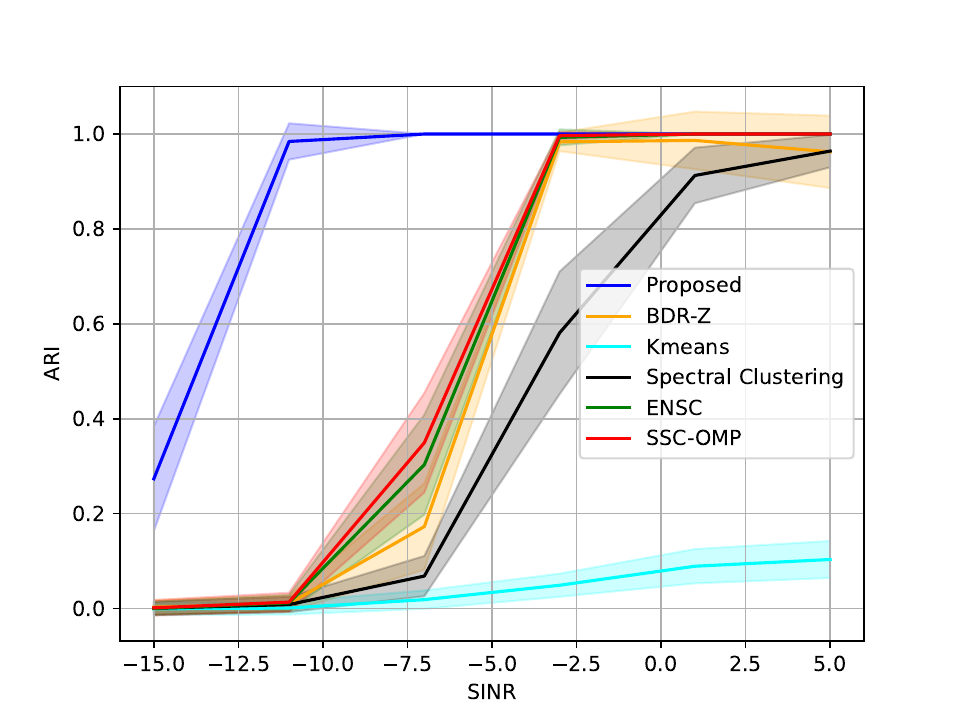}
    \includegraphics[width=0.42\textwidth, trim=20 10 30 10, clip]{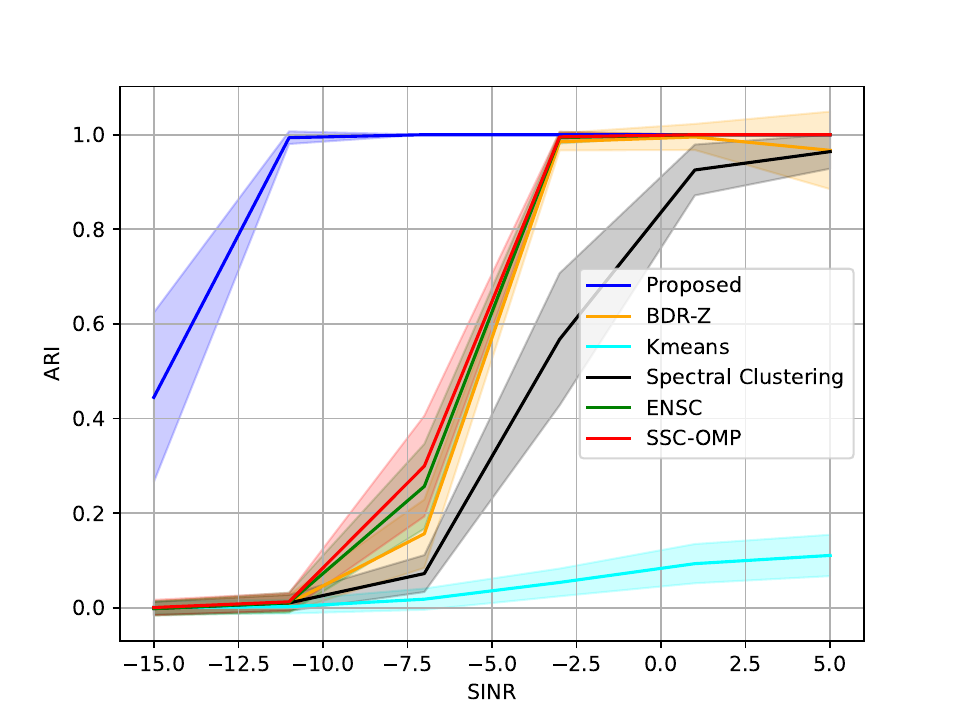}
    \includegraphics[width=0.42\textwidth, trim=20 10 30 10, clip]{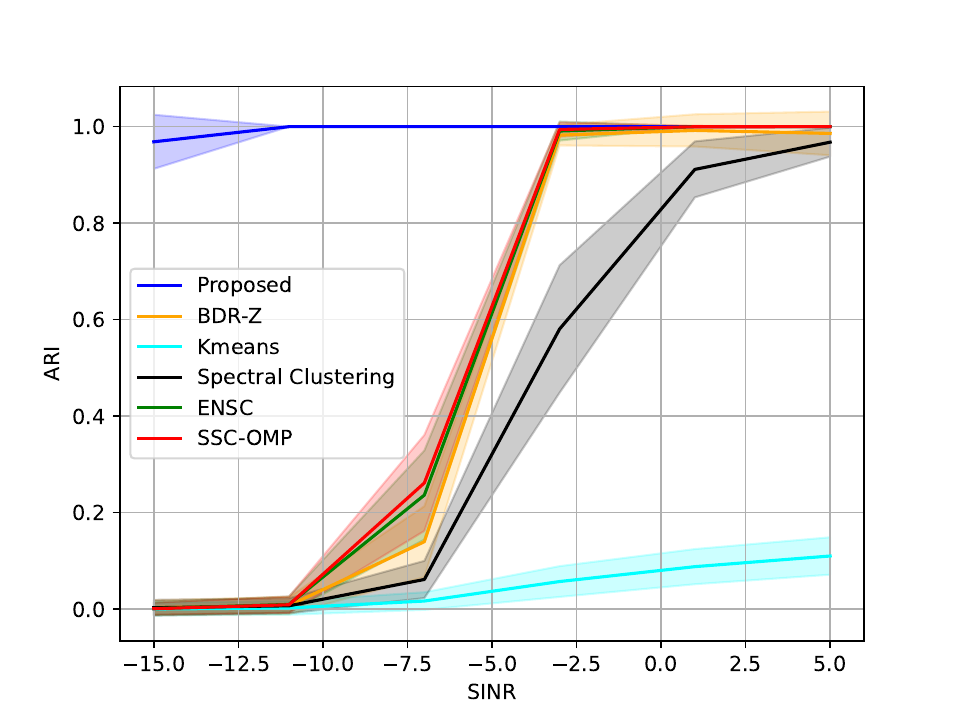}
    \includegraphics[width=0.42\textwidth, trim=20 10 30 10, clip]{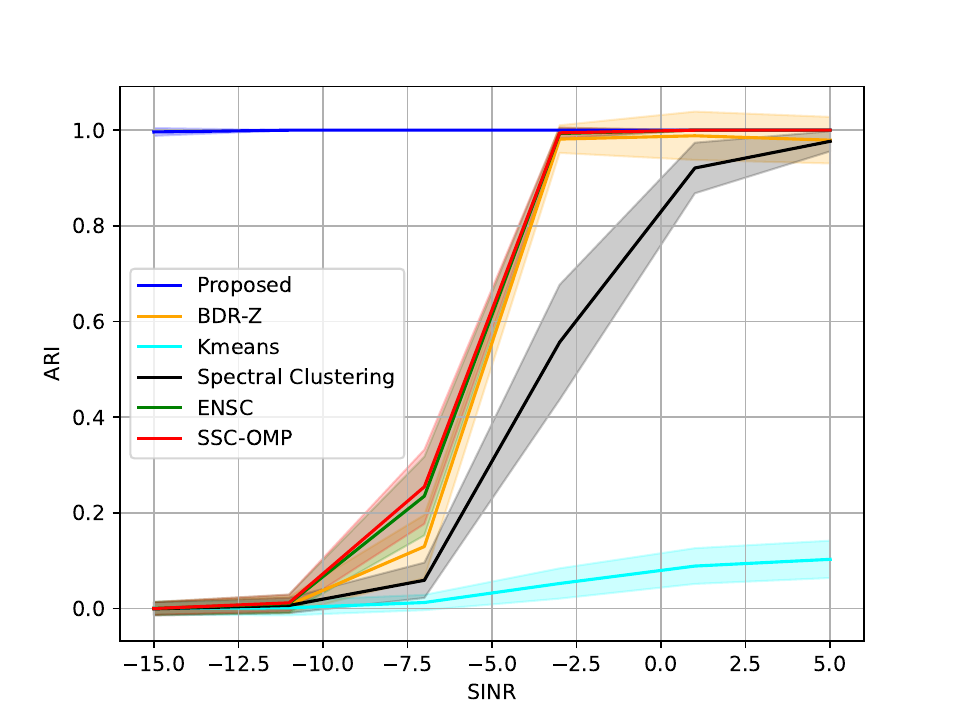}
    \caption{Average ARI metric over SINR for INR = $\left\{0.5,~1,~5,~10\right\}$ (starting from above), based on $100$ Monte Carlo runs.}
    \label{fig:ari}
\end{figure}

\begin{figure}[t]
    \centering
    \includegraphics[width=0.42\textwidth, trim=20 10 30 10, clip]{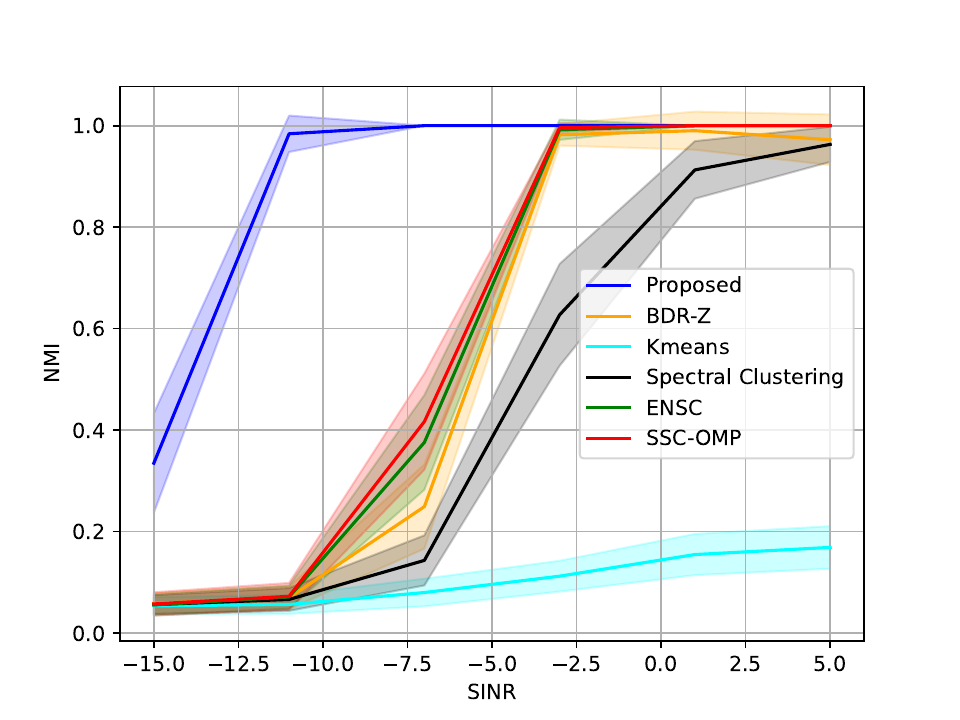}
    \includegraphics[width=0.42\textwidth, trim=20 10 30 10, clip]{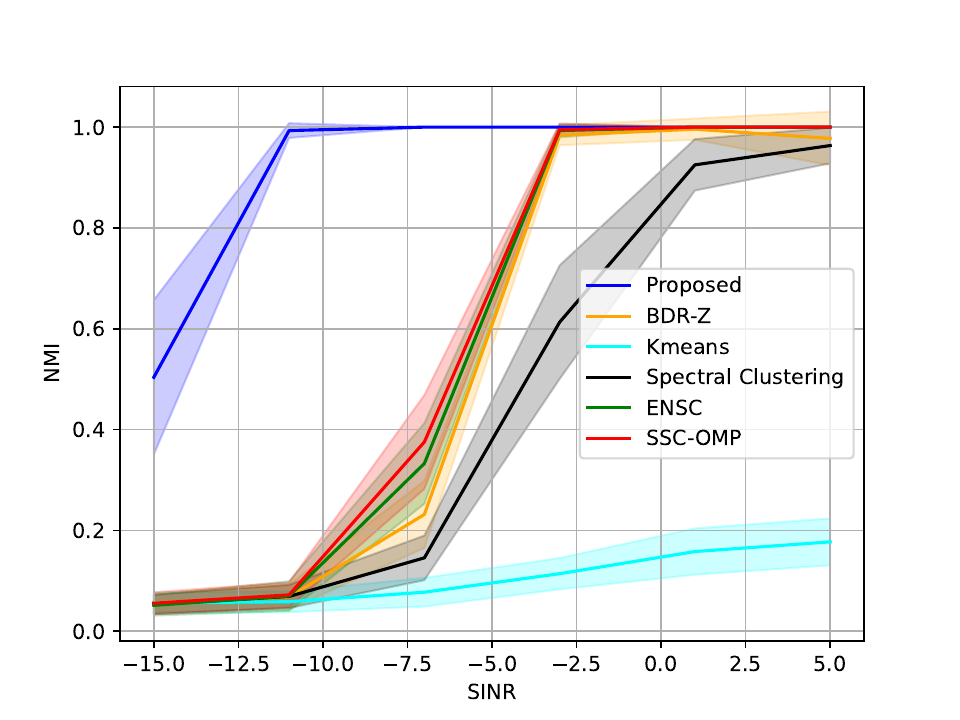}
    \includegraphics[width=0.42\textwidth, trim=20 10 30 10, clip]{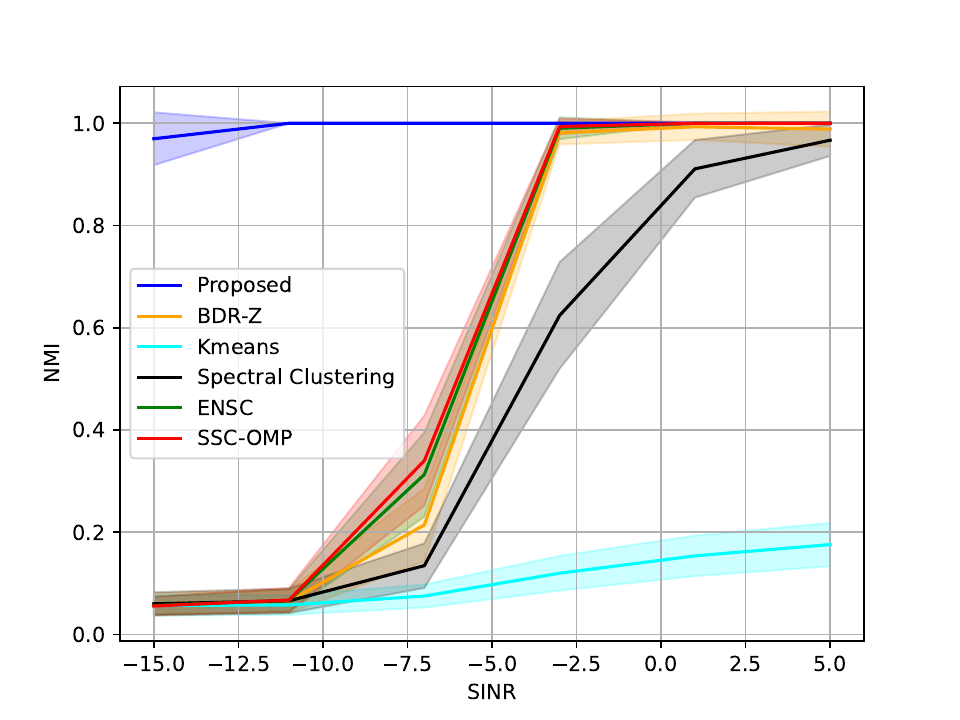}
    \includegraphics[width=0.42\textwidth, trim=20 10 30 10, clip]{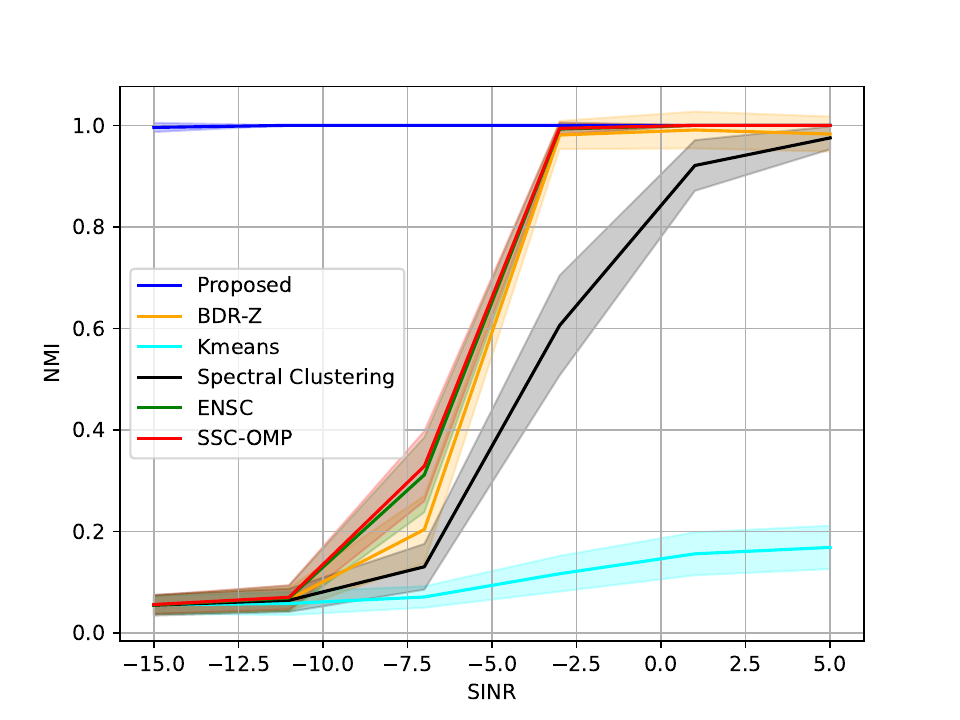}
    \caption{Average NMI metric over SINR for INR = $\left\{0.5,~1,~5,~10\right\}$ (starting from above), based on $100$ Monte Carlo runs.}
    \label{fig:nmi}
\end{figure}

\begin{figure}[t]
    \centering
    \includegraphics[width=0.42\textwidth, trim=20 10 30 10, clip]{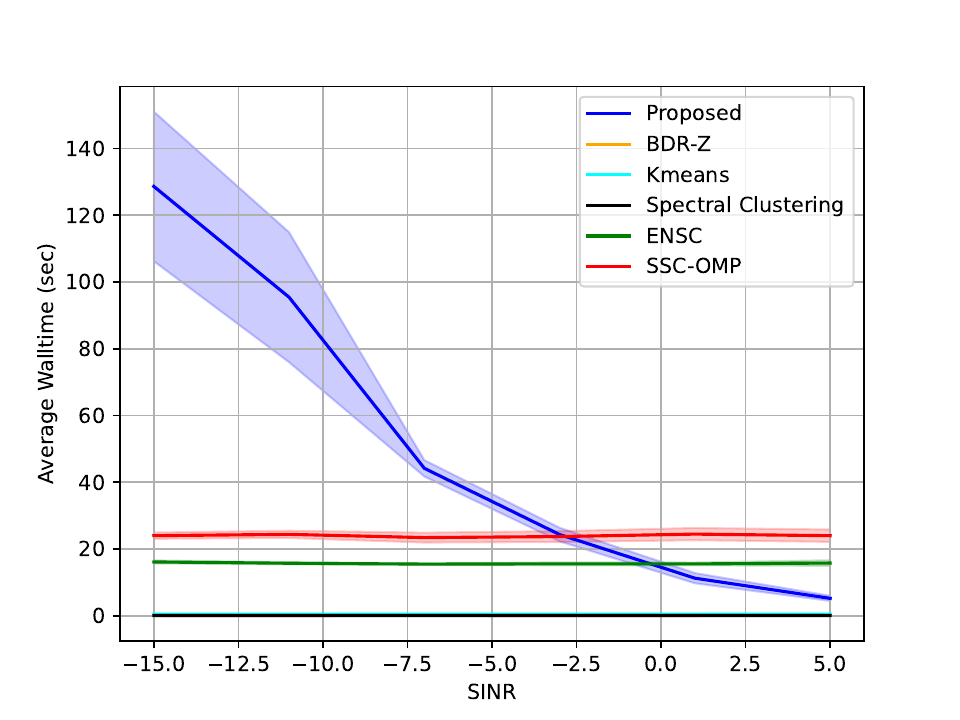}
    \includegraphics[width=0.42\textwidth, trim=20 10 30 10, clip]{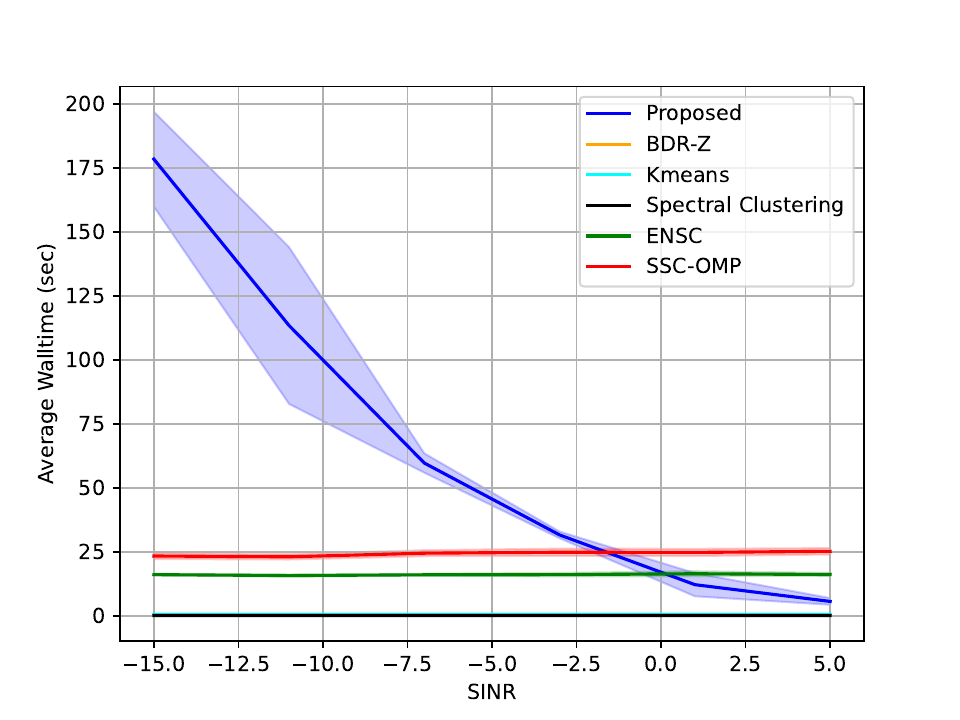}
    \includegraphics[width=0.42\textwidth, trim=20 10 30 10, clip]{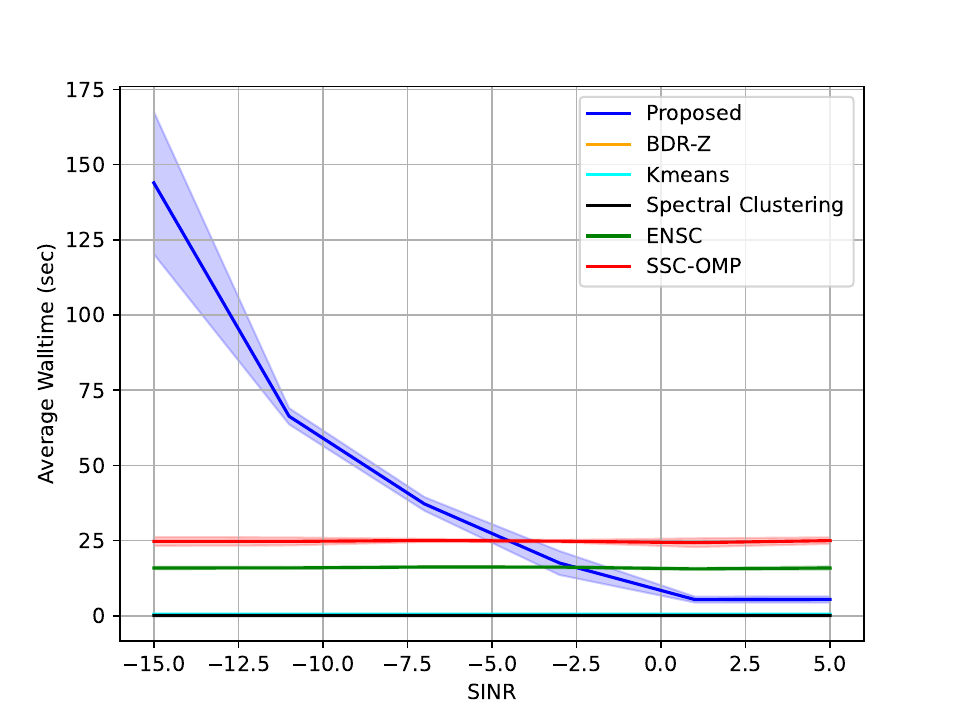}
    \includegraphics[width=0.42\textwidth, trim=20 10 30 10, clip]{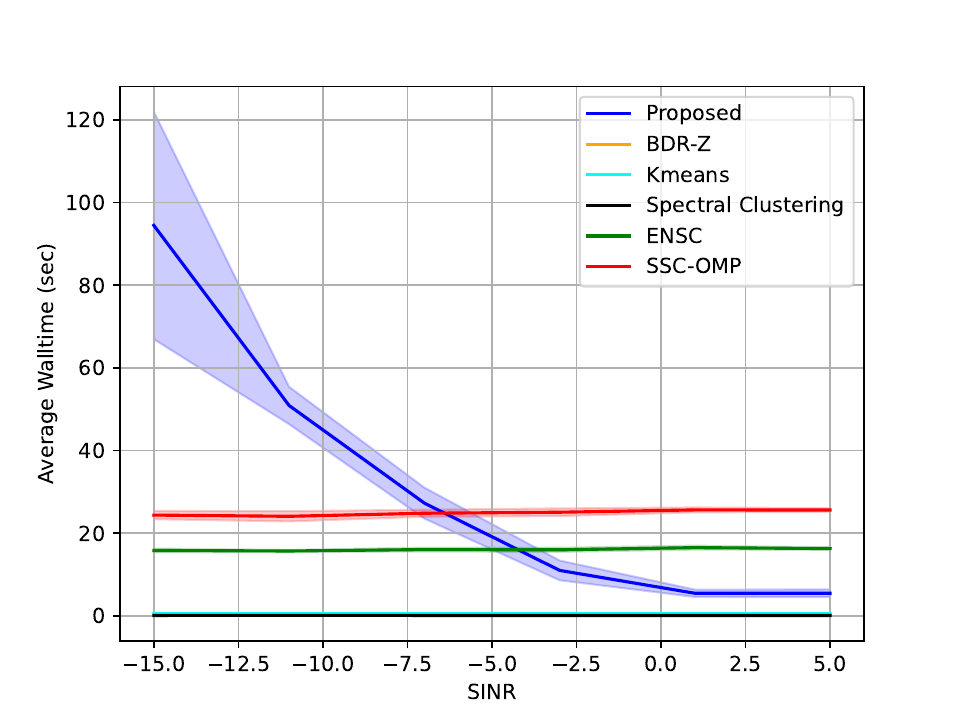}
    
    \caption{Average walltimes over SINR for INR = $\left\{0.5,~1,~5,~10\right\}$ (starting from above), based on $100$ Monte Carlo runs.}
    \label{fig:times}
\end{figure}

\end{document}